\documentclass[twoside]{article}

\usepackage[accepted]{aistats2025}
\usepackage[round]{natbib}

\usepackage[utf8]{inputenc} 
\usepackage[T1]{fontenc}    
\usepackage{hyperref}       
\usepackage{url}            
\usepackage{booktabs}       
\usepackage{amsfonts}       
\usepackage{nicefrac}       
\usepackage{microtype}      

\usepackage{mdwlist}

\usepackage{graphicx}
\usepackage{booktabs,colortbl} 
\usepackage{wrapfig}
\usepackage{caption}
\usepackage{subcaption}
\usepackage[dvipsnames]{xcolor}
\newcommand{\grn}[1]{{\color{Green} #1}}

\usepackage[textsize=tiny]{todonotes}

\usepackage{amsmath}
\usepackage{amssymb}
\usepackage{mathtools}
\usepackage{amsthm}
\usepackage[mathscr]{eucal}
\usepackage{mathrsfs}

\usepackage{balance}

\theoremstyle{plain}
\newtheorem{theorem}{Theorem}[section]
\newtheorem{proposition}[theorem]{Proposition}
\newtheorem{lemma}[theorem]{Lemma}

\theoremstyle{definition}
\newtheorem{definition}[theorem]{Definition}

\newcommand{\eq}[1]{\begin{equation}{#1}\end{equation}}

\newcommand{\al}[1]{\begin{align}{#1}\end{align}}
\newcommand{\prn}[1]{\left({#1}\right)}
\newcommand{\brt}[1]{\left[{#1}\right]}
\newcommand{\abs}[1]{\left|{#1}\right|}
\newcommand{\nrm}[1]{\left|\left|{#1}\right|\right|}
\newcommand{\mbf}[1]{\mathbf{#1}}
\newcommand{\mbb}[1]{\mathbb{#1}}
\newcommand{\mc}[1]{\mathcal{#1}}
\newcommand{\mf}[1]{\mathfrak{#1}}
\newcommand{\ms}[1]{\mathscr{#1}}
\newcommand{\bs}[1]{\boldsymbol{#1}}

\begin{document}

%

%

\twocolumn[

\aistatstitle{Learning Laplacian Positional Encodings for Heterophilous Graphs}

\aistatsauthor{ Michael Ito$^1$ \And Jiong Zhu$^1$ \And  Dexiong Chen$^2$ \And Danai Koutra$^1$ \And Jenna Wiens$^1$ }


\aistatsaddress{ University of Michigan$^1$, Max Planck Institute of Biochemistry$^2$ } ]

\begin{abstract}
In this work, we theoretically demonstrate that current graph positional encodings (PEs) are not beneficial and could potentially hurt performance in tasks involving heterophilous graphs, where nodes that are close tend to have different labels. This limitation is critical as many real-world networks exhibit heterophily, and even highly homophilous graphs can contain local regions of strong heterophily. To address this limitation, we propose Learnable Laplacian Positional Encodings (LLPE), a new PE that leverages the full spectrum of the graph Laplacian, enabling them to capture graph structure on both homophilous and heterophilous graphs. Theoretically, we prove LLPE's ability to approximate a general class of graph distances and demonstrate its generalization properties. Empirically, our evaluation on 12 benchmarks demonstrates that LLPE improves accuracy across a variety of GNNs, including graph transformers, by up to 35\% and 14\% on synthetic and real-world graphs, respectively. Going forward, our work represents a significant step towards developing PEs that effectively capture complex structures in heterophilous graphs.

\end{abstract}

\section{INTRODUCTION}

\begin{figure}[t]
    \centering
    \includegraphics[width=0.5\textwidth]{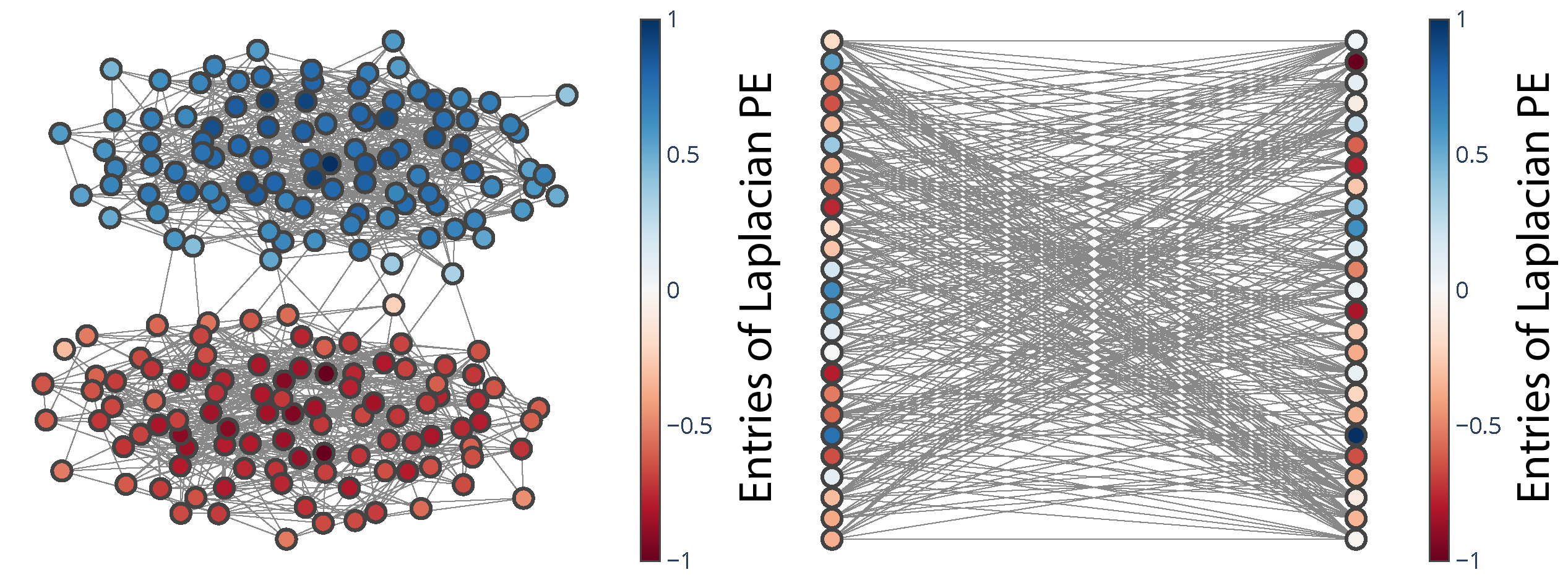}
    \caption{\small LPEs on homophilous/heterophilous SBMs.}
    \label{fig:sbms}
\end{figure}

In node classification, graph positional encodings (PEs) improve the discriminative performance of graph neural networks (GNNs) by injecting them with valuable positional information, allowing them to better capture the positions of nodes within the graph~\citep{you2019position, dwivedi2022graph, rampavsek2022recipe}. In particular, PEs aim to encode distance such that if two nodes are close in graph distance, their PEs should also be close. While it has been shown that PEs are beneficial for homophilous graphs, where nodes of the same label tend to be close, we find that they are not as beneficial in heterophilous graphs, where nodes that are close tend to have different labels. To demonstrate this intuition, consider a homophilous graph of two clusters where the cluster assignment is the node label (Figure \ref{fig:sbms}). Here, Laplacian PEs that encode closeness by leveraging the eigenvectors of the Laplacian are beneficial since nodes that are close have the same label. Now consider a heterophilous bipartite graph where node labels are the sets that do not connect to each other. Here, nodes that are close have opposing labels, and thus Laplacian PEs no longer capture the relevant structural information. 

Notably, real-world graphs exhibit a spectrum of homophily-heterophily, and while many real-world graphs exhibit strong homophily, there are also many graphs that exhibit strong heterophily such as protein networks where different amino acids form connections and transaction networks where fraudsters are more likely to connect to accomplices \citep{zhu2020beyond}. Furthermore, we find that even if a graph is highly homophilous globally, it may still contain local regions that are highly heterophilous as seen in some social networks, where users can form diverse connections \citep{newman2018networks,LovelandZHFSK23discrepancies}.
These observations have led to significant research efforts in understanding the impact of heterophily on GNNs and designing models that can generalize to both homophilous and heterophilous graphs~\citep{zhu2020beyond,chien2021adaptive, bo2021beyond, zhu2021graph, yan2022two, ma2022is,luan2024heterophilicgraphlearninghandbook, ZhuLY0CK24linkpred,ItoKW25dynamic}. In this work, given the spectrum of homophily-heterophily in real data, we investigate the \textit{impact of PEs on node classification in both homophilous and heterophilous graph settings}. While orthogonal to the design of heterophilous GNNs, our work can improve these models by augmenting them with our proposed PEs.

In our theoretical analysis, we focus primarily on Laplacian PEs (LPEs) due to their importance and prevalence in GNN designs \citep{rampavsek2022recipe, dwivedi2021generalization, kreuzer2021rethinking, lim2022sign, wang2022equivariant}. Our analysis reveals that LPEs can fail to capture relevant structures on heterophilous 
graphs since they do not include the full spectrum of the Laplacian. While other learnable PEs such as RWSE \citep{dwivedi2022graph} and SAN-PE \citep{kreuzer2021rethinking} can be extended to the full spectrum, empirically we find that they often struggle with high dimensional and/or noisy eigenspaces and as a result can fail to capture relevant structures (i.e., identify the relevant eigenvectors of the Laplacian). To address this gap, we introduce Learnable Laplacian Positional Encodings (LLPEs), a new PE that leverages the relevant eigenvectors and eigenvalues of the graph Laplacian. We demonstrate LLPE's theoretical expressivity by showing its ability to approximate a general class of graph distances, enabling it to capture relevant graph structure on arbitrary graphs. We further prove that it exhibits better statistical generalization properties compared to other designs. Empirically, we demonstrate the effectiveness of LLPE by showing that it improves performance for a variety of GNNs, including graph transformers, in comparison to other PEs across 12 homophilous and heterophilous benchmarks. In summary, we make the following contributions.

\begin{itemize}
    \item \textbf{Identification of Limitations.} We show that LPEs and popular learnable PEs can fail to capture relevant graph structure in heterophilous graphs since they do not scale to high dimensions. 

    \item \textbf{New Position Encodings for Homophily and Heterophily.} We introduce Learnable Laplacian Positional Encodings that capture relevant structure on homophilous and heterophilous graphs by leveraging the full spectrum of the Laplacian.

    \item \textbf{Theoretical Insights.} We show theoretically that LLPE can approximate important notions of graph distance and that it exhibits the best statistical generalization among other designs. 
    
    \item \textbf{Empirical Analysis.} Empirically, we show that LLPE improves performance over current popular encodings on 12 node classification benchmarks. 
\end{itemize}

\section{BACKGROUND}

In this section, we first provide background on node classification since it is the main task we consider throughout our work. We then provide an overview of message-passing GNNs and their relationship to homophily and heterophily. Next, we provide background on graph PEs, which are key components for many GNNs in both node and graph classification. Lastly, we provide background on graph transformers, a GNN architecture that has achieved high performance across a variety of graph benchmarks due to the design of effective PEs \citep{rampavsek2022recipe, dwivedi2021generalization, kreuzer2021rethinking, chen2022structure}.

\subsection{Node Classification}

A graph is defined $G = (V, \textbf{A}, \mbf{X}, \mbf{y})$, where $V$ is the node set, $\mbf{A} \in \mbb{R}^{\abs{V} \times \abs{V}}$ is the adjacency matrix, $\mbf{X} \in \mbb{R}^{\abs{V} \times d}$ is the node feature matrix, and $\mbf{y} \in \mbb{R}^{\abs{V}}$ is node label vector. For node $i \in V$, we denote its feature vector and label as $\mbf{x}_i$ and $y_i$, respectively. Given a single $G$ with a random sample of node representations $\mbf{X}_\text{tr} = \{\mbf{x}_0, \ldots, \mbf{x}_{n_{\text{tr}}}\}$ and their labels $\mbf{y}_\text{tr} = \{y_0, \ldots, y_{n_{\text{tr}}}\}$, the node classification task is to learn a classifier $f$ such that the expected loss $\mbb{E}[\ms{L}(f(\textbf{x}, \mbf{A}), y)]$ is minimized, where $\ms{L}$ is some loss function. Given $f$, one can then label the remaining nodes in $G$. In our work, we focus on the transductive setting where no new nodes are added at inference.

\subsection{Homophily, GNNs and Graph PEs}

\textbf{Graph Homophily.}
The homophily ratio is the probability that a node forms an edge with another node with the same label, and a graph is considered homophilous if it has a high homophily ratio and heterophilous otherwise~\citep{zhu2020beyond}. 

\textbf{Message-Passing GNNs.} Many GNNs follow a message-passing scheme, where each layer updates each node's representation by aggregating the representations of its immediate neighbors \citep{gilmer2017neural}. Most message-passing GNNs differ in the aggregation step, and popular choices include the symmetric normalized mean \citep{kipf2017semi}, attention \citep{velickovic2018graph}, and sum pooling \citep{xu2018how}. Intuitively, when many nodes are connected to other nodes with the same label, message-passing is beneficial, and different classes will be well-separated in feature space after message passing. On the other hand, when many nodes are connected to nodes of different labels, message-passing may be detrimental and result in significant overlap between different classes in feature space. Notions of homophily attempt to capture this phenomenon~\citep{MironovP24measures}. 

In recent years, new GNN designs have been proposed in order to make these models effective for heterophilous graphs, without compromising performance over homophilous graphs, including separation of ego- and neighbor- embeddings~\citep{zhu2020beyond}, learning from neighbors at various distances~\citep{abu2019mixhop}, learned degree corrections~\citep{yan2022two}, and more~\citep{chien2021adaptive, bo2021beyond, zhu2021graph, yan2022two, ma2022is,luan2024heterophilicgraphlearninghandbook}. Despite these advances, GNNs are known to exhibit various limitations, as they lack the ability to encode the position of a node within a graph \citep{you2019position}. Thus, they are commonly augmented with PEs \citep{kreuzer2021rethinking}. 

\textbf{Transformer-based GNNs.} Graph transformers (GTs) are a new class of GNN, replacing message-passing in favor of global attention \citep{chen2022structure, ma2023graph}. Due to their departure from message-passing, GTs are believed to overcome many of the limitations of message-passing and as a result have achieved high performance on a variety of benchmarks \citep{kreuzer2021rethinking, rampavsek2022recipe}. Moreover, a crucial component of GTs is the choice of PE since it is the main component of the architecture that captures graph structure \citep{dwivedi2021generalization}. 

\textbf{Graph Positional Encodings.}  In the context of learning from graphs, PEs are typically added or concatenated with node features. In our theoretical analysis, we focus on Laplacian positional encoding (LPE) defined as the $k$ eigenvectors of the graph Laplacian that correspond to its $k$ smallest eigenvalues. LPEs have been shown to improve GNNs both theoretically and empirically by encoding useful notions of node positional information. 

\section{ANALYSIS OF LPEs \label{sec:3}}

In this section, we explore LPEs in capturing relevant graph structure on stochastic block models (SBMs) \citep{holland1983stochastic}. In our analysis, we show that LPEs do capture the relevant graph structure in homophilous SBMs but fail to do so in heterophilous SBMs. Instead, the eigenvectors corresponding to the largest eigenvalues capture the relevant graph structure. For the remainder of our analysis, we assume that the eigenvalues are sorted from smallest to largest, and the corresponding eigenvectors are sorted similarly. Our analysis both highlights the limitations of LPEs and sheds light on potential solutions to these limitations.

\subsection{Community Detection on SBMs}

We consider $G(n, k, p, q)$, a multiclass SBM, where the set of $n$ vertices are divided into $k$ communities of size $\frac{n}{k}$. Edges between nodes in the same community are sampled with probability $p$, while edges between nodes in different communities are sampled with probability $q$. The community detection task is to recover the community labels for all nodes given a single realization of $G$. If LPEs can recover the communities, they effectively capture the relevant graph structure since the communities provide the best notion of node position within the SBM. Importantly, an SBM is homophilous when $p > q$ and heterophilous otherwise. Thus, by specifying conditions on $p$ and $q$, we can analyze LPEs along different homophilous and heterophilous graphs.

\subsection{Laplacian Encodings on Multiclass SBMs}

To put our theoretical results into context, we first restate the following well-known result from the network community detection literature that states in a homophilous multiclass SBM, the first $k$ eigenvectors of the graph Laplacian recovers the node communities up to a small error \citep{abbe2018community}. 

\begin{theorem}[\citet{abbe2018community}]
Let $\mbf{A}$ and $\mbf{L}$ be the adjacency and Laplacian matrix drawn from the stochastic block model $G(n, k, p, q)$. Assume $p \gg q$ and $\text{min}(d_i) \geq C \text{ln}(n)$ where $C$ is an appropriately large constant. Then, with high probability, the nonzero entries along the rows of the \textbf{first nontrivial $k-1$ eigenvectors} of $\mbf{L}$ correctly recover the true communities up to an orthogonal transformation with at most $\mc{O}(k^\frac{3}{2})$ misclassified nodes.
\label{thm:1}
\end{theorem}

Since nearby nodes have the same label, LPEs capture relevant graph structure in homophilous SBMs. We now pose the same question for heterophilous SBMs where nearby nodes have different labels. In our next theorem, we demonstrate that on heterophilous SBMs the first $k$ eigenvectors no longer recover the communities, and instead, the last $k$ eigenvectors do.

\begin{theorem}
Let $\mbf{A}$ and $\mbf{L}$ be the adjacency and Laplacian matrix drawn from $G(n, k, p, q)$. Assume $q \gg p$ and $\text{min}(d_i) \geq C \text{ln}(n)$ where $C$ is an appropriately large constant. Then, with high probability, the nonzero entries along the rows of the \textbf{last $k-1$ eigenvectors} of $\,\mbf{L}$ correctly recover the true communities up to an orthogonal transformation with at most $\mc{O}(k^\frac{3}{2})$ misclassified nodes. Moreover, the first nontrivial $k$ eigenvectors \textbf{do not} recover the true communities.
\label{thm:2}
\end{theorem}

We prove Theorem \ref{thm:2} in Appendix \ref{pf:lap}. Since on heterophilous SBMs, nodes that are close have opposing labels, LPEs do not capture the relevant structure and in order to do so the last $k - 1$ eigenvectors need to be leveraged. This result captures the same intuition as designs in heterophilous GNNs \citep{zhu2020beyond} aiming to capture high frequency components of the graph by leveraging intermediate GNN representations. 


\section{LEARNABLE LAPLACIAN ENCODINGS}

Our investigation of LPEs finds that community structure in heterophilous graphs is not captured by the first $k$ eigenvectors, but rather the last $k-1$. Since graphs can be homophilous or heterophilous, we propose \textit{learning} which parts of the spectrum are important. In this section, we first introduce Learnable Laplacian Position Encodings (LLPE), a new PE that leverages the full spectrum of the graph Laplacian. We demonstrate LLPE's theoretical expressivity by showing it captures relevant graph structure on both homophilous and heterophilous graphs as well as general notions of graph distance for arbitrary graphs. We next show that LLPE exhibits the best statistical generalization among other design choices, justifying our designs. We then discuss an extension of LLPE to large graphs when the full eigendecomposition is infeasible.

\subsection{Learnable Laplacian Position Encodings (LLPE)}

Given the eigendecomposition of the graph Laplacian, $\mbf{L} = \mbf{U}^\top \mbf{\Lambda} \mbf{U}$, LPEs can be represented as follows: $\mbf{P}_{\text{LPE}} = \mbf{U}_{k} \mbf{W}$, where $\mbf{U}_{k}$ are the first $k$ eigenvectors and $\mbf{W}$ is a learnable linear projection matrix. As shown in Theorem \ref{thm:2} the first $k$ eigenvectors may not capture heterophilous graph structure since heterophily may be captured in other parts of the spectrum. Current approaches apply MLPs or transformers to only the first $k$ eigenvectors, and one simple approach is to extend these PEs to include the full eigenvector matrix $\mbf{U}$ rather than only the first $k$ eigenvectors. However, we find that this approach does not work well in practice since intuitively we expect that only a small number of eigenvectors are relevant, and the dimension of $\mbf{U}$ can be very large especially for large graphs. We further note that LPEs do not explicitly leverage eigenvalue information. We construct LLPE to address these limitations.

\begin{figure}[t!]
    \centering
    \includegraphics[width=0.5\textwidth]{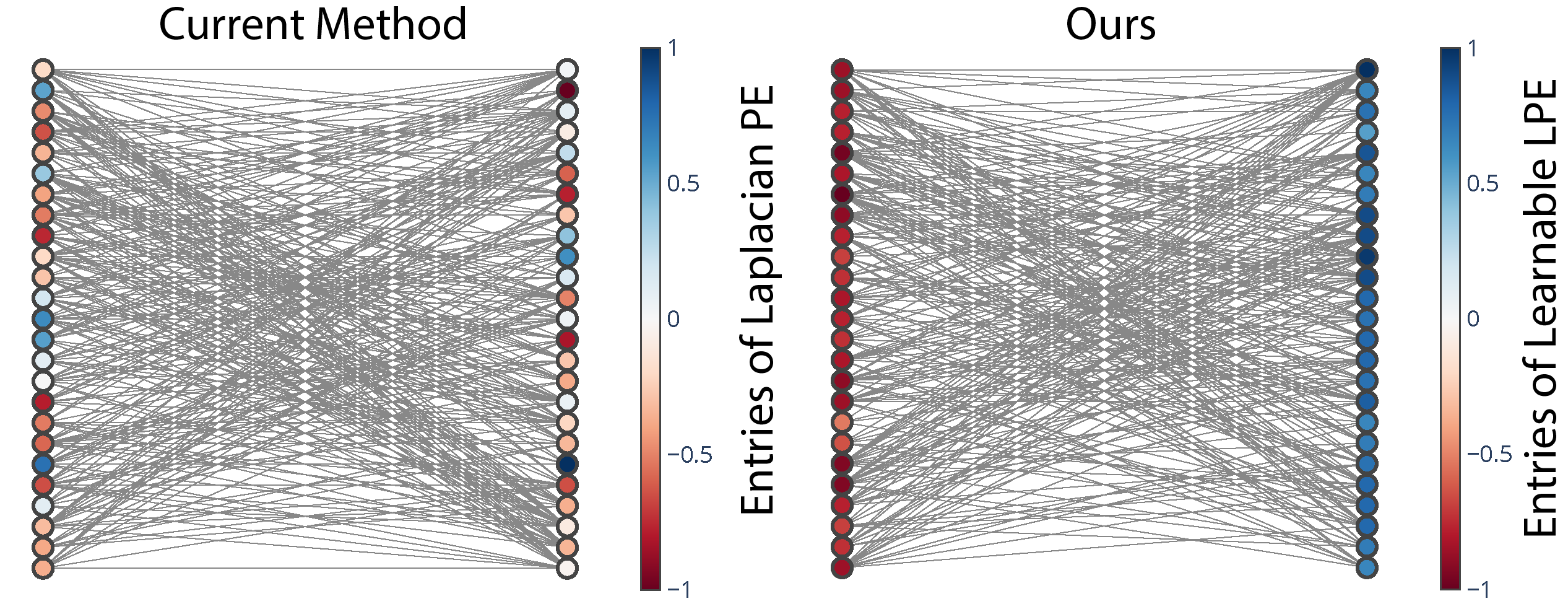}
    \caption{\small LPEs and LLPEs on heterophilous SBMs.}
    \label{fig:sbm_llpe}
\end{figure}

The key insight in the design of LLPE is to leverage the full eigenvector matrix $\mbf{U}$ along with the corresponding eigenvalues $\mbf{\Lambda}$. By leveraging the eigenvalue information, we can learn which eigenvectors are important. More specifically, we learn a mapping $h: [0, 2] \to \mbb{R}$ where $h(\lambda_i)$ represents eigenvector $i$'s importance in the PE. Formally, LLPE can be defined as:  
\eq{\mbf{P}_{\text{LLPE}} = \mbf{U} \mbf{W}_{\text{LLPE}} \text{, where}}
\eq{\mbf{W}_{\text{LLPE}} = \begin{pmatrix}
h(\lambda_1; \bs{\theta}_1) & \cdots & h(\lambda_1; \bs{\theta}_d)\\
\vdots & \ddots & \vdots \\
h(\lambda_n; \bs{\theta}_1) & \cdots & h(\lambda_n; \bs{\theta}_d)\\
\end{pmatrix}, }

where $\mbf{W}_{\text{LLPE}} \in \mbb{R}^{n \times d}$ and $\bs{\theta}_j$ is a weight vector that parametrizes $h$. Inspired by their use in spectral filtering \citep{hammond2011wavelets}, we set $h$ to be a truncated Chebyshev series. Thus, $h(\lambda_i; \bs{\theta}_j)$ has the form 
\eq{h(\lambda_i; \bs{\theta}_j) = \sum_{m=0}^M \bs{\theta}_j[m] \cdot T_m(\tilde \lambda_i), \text{ where}}
\eq{T_m(\tilde \lambda_i) = \text{cos}(m \cdot \text{arccos}(\tilde \lambda_i)),}

where $\tilde \lambda$ are the normalized eigenvalues, $M$ is a tunable hyperparameter, $T_m$ is a Chebyshev polynomial of order $m$, and $\bs{\theta}_1, \cdots,\bs{\theta}_d \in \mbb{R}^M$ are learnable Chebyshev coefficients. To learn these parameters, we update them via backpropagation when training the GNN. We additionally regularize the $l^1$ and $l^2$ norms of the columns of $\mbf{W}_{\text{LLPE}}$ during training to encourage sparse outputs. Figure \ref{fig:sbm_llpe} demonstrates an overview of LLPE's improvements on community detection over LPE on binary heterophilous SBMs.

\subsection{Theoretical Expressivity of LLPE}

In this section, we demonstrate LLPE's theoretical expressivity. We first show that LLPE can capture relevant structure on homophilous and heterophilous SBMs. We next demonstrate that LLPE can approximate general notions of graph distances, including ones defined with random walks, heat kernels, and diffusion maps. This shows that LLPE can recover relevant structure on arbitrary graphs. In the following proposition, we state the approximation result for SBMs.

\begin{proposition}
Let $\mbf{A}$ and $\mbf{L}$ be the adjacency and Laplacian matrix drawn from $G(n, k, p, q)$. If $p \gg q$ or $q \gg p$ and $\text{min}(d_i) \geq C \text{ln}(n)$ where $C$ is an appropriately large constant, then LLPE can correctly recover the true communities up to an orthogonal transformation with at most $\mc{O}(k^{\frac{3}{2}})$ misclassified nodes.
\label{prp:com}
\end{proposition}

We prove proposition \ref{prp:com} in Appendix \ref{pf:llpe_app}. Since LLPE leverages the full eigenvector matrix $\mbf{U}$ and its corresponding eigenvalues $\mbf{\Lambda}$, it can capture the graph structure on both homophilous and heterophilous SBMs. We next show that LLPE can further capture relevant notions of structure for arbitrary graphs. We begin by introducing a general class of functions on graphs.

\begin{definition}
Let $G$ be an arbitrary graph. Define $f_r: V \times V \to \mbb{R}$ as a function on $G$ with respect to $r: [0, 2] \to \mbb{R^+}$ such that $f_r(i, j)$ has the following form,
\eq{f_r(i, j)^2 = \sum_{k=1}^n r(\lambda_k) (\mbf{u}_k[i] - \mbf{u}_k[j])^2.}
\label{def:dist}
\vspace{-1em}
\end{definition}

When $r$ is monotonically decreasing in $\lambda$, Definition \ref{def:dist} generalizes a variety of distances on graphs. For example when $r(\lambda_k) = \frac{1}{\lambda_k^2}$, $f_r$ is the commute time \citep{lovasz1993random}, where $f_r(i, j)$ is the expected number of steps in a random walk starting at node $i$ travelling to node $j$ then travelling back to node $i$. When $r(\lambda_k) = e^{-2t\lambda_k}$ for some parameter $t$, $f_r$ is the diffusion distance \citep{coifman2006diffusion}, where $f_r(i, j)$ is the number of paths of length $t$ from node $i$ to node $j$. Finally, when $r(\lambda_k) = \frac{1}{\lambda_k}$, $f_r$ is the Biharmonic distance \citep{lipman2010biharmonic} proposed to improve upon the commute time and diffusion distance.

In the cases of the commute time, diffusion distance, and Biharmonic distance, $r$ acts as a low pass filter that amplifies the first eigenvectors and attenuates the last ones. As a result, nodes that are close will have small $f_r(i, j)$, while nodes that are far away will have large $f_r(i, j)$, effectively capturing a notion of distance. For homophilous graphs, these notions capture the relevant graph structure. On the other hand, if $r$ is monotonically increasing in $\lambda$, then $r$ acts as a high pass filter, and nodes far away will have small $f_r(i, j)$, while nodes that are close will have large $f_r(i, j)$. In this case, $f_r(i, j)$ captures the relevant graph structure on heterophilous graphs. 

LLPE can approximate functions on graphs of the form of Definition \ref{def:dist} for any $r$, demonstrating that LLPE can capture relevant notions of graph structure. We prove Theorem \ref{thm:dst} in Appendix \ref{pf:llpe_app}.

\begin{theorem}
Let $G$ be an arbitrary graph and $f_r$ be a function on $G$ of the form in Definition \ref{def:dist} for some $r$. LLPE can recover $f_r$ such that for any nodes $i$ and $j$, the $l^2$ distance between LLPE's encoding for nodes $i$ and $j$ approximates $f_r(i, j)$.
\label{thm:dst}
\end{theorem}

\subsection{Generalization Properties of LLPE}

In this section, we discuss the statistical generalization of LLPE. Our main result derives the upper and lower bounds of LLPE's Rademacher complexity, indicating that generalization depends on the norms of the Chebyshev coefficients. Our analysis justifies our choice of $h$. In particular, defining $h$ as a truncated Chebyshev series obtains better generalization over other approximating polynomials and existing PEs. 

\begin{theorem}
Let $\mc{H}_{\text{LLPE}} = \{\tilde\lambda \to \sum_{m=1}^M \theta_m \cdot \tilde T_m(\tilde \lambda) : \bs{\theta} \in \mbb{R}^M, \nrm{\bs{\theta}}_2 \leq C_{\text{LLPE}}\}$, where $C_{\text{LLPE}}$ is some constant greater than 0, $\tilde\lambda$ denotes the normalized eigenvalues, and $\tilde T_m$ is the normalized Chebyshev polynomial of order $m$. Then, the empirical Rademacher complexity of the hypothesis class $\mc{H}_{\text{LLPE}}$ for a sample $S = (\lambda_1, \ldots, \lambda_n)$ admits the following upper and lower bounds: 
\eq{\frac{C_{\text{LLPE}}}{\sqrt{2n}} \leq \hat{\mf{R}}_{S}(\mc{H}_{\text{LLPE}}) \leq \frac{\sqrt{2}C_{\text{LLPE}}}{\sqrt{n}}}
\label{thm:rad}
\vspace{-1em}
\end{theorem}

We prove Theorem \ref{thm:rad} in Appendix \ref{pf:llpe_gen}. Theorem \ref{thm:rad} tells us that the empirical Rademacher complexity scales with the upper bound of the $l^2$ norm of the Chebyshev coefficients. Thus, the bounds do not depend \textit{explicitly} on the order $M$ of the Chebsyshev series, but rather implicitly. As a result, LLPE obtains high expressivity and good generalization with large $M$, so long as the norm of the Chebyshev coefficients remains small. To prove this result, we leverage the fact that Chebyshev polynomials of order $m$ obtain the minimum $l^\infty$ norm among all other polynomials of order $m$. Importantly, the extension of other learnable PEs that rely on applying an MLP or transformer to the full eigenvector matrix $\mbf{U} \in \mbb{R}^{n \times n}$ result in model weights with dimension scaling in $n$, the number of nodes, which can be very large for large graphs. The dimension of LLPE's Chebyshev coefficients do not depend on $n$, but rather rely on $M$ the number of terms in the sum (which is typically much smaller than $n$). This serves as a form of implicit regularization for large graphs, and provides theoretical insight into why LLPE is able to learn relevant eigenvectors in high-dimensional noisy graphs while other learnable PEs cannot.

\begin{figure*}[t!]
     \centering
     \begin{subfigure}[b]{.49\textwidth}
         \centering
         \includegraphics[width=\textwidth]{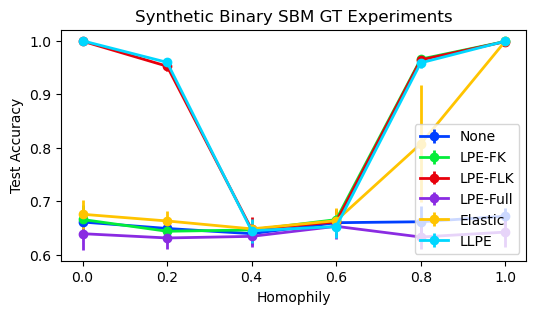}
         \caption{Binary SBM Experiment}
         \label{fig:eu}
     \end{subfigure}
     \hfill
     \begin{subfigure}[b]{.49\textwidth}
         \centering
         \includegraphics[width=\textwidth]{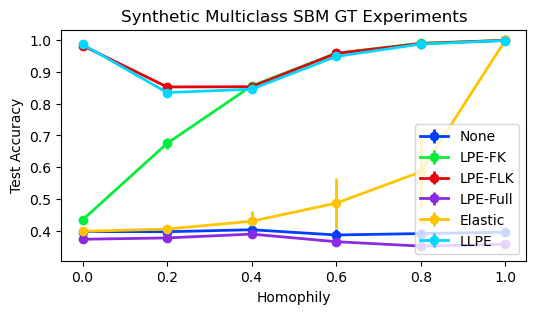}
         \caption{Multiclass SBM Experiment}
         \label{fig:eu}
     \end{subfigure}
    \caption{Mean and standard deviations (error bars) of all model-PE combinations on the synthetic SBMs. LLPE performs well across both high homophily and high heterophily, while LPE-FK does not.}
    \label{fig:sbm}
\end{figure*}

\subsection{Extension of LLPE to Large Graphs} 

Our discussion of LLPE thus far has assumed access to the full eigendecomposition of the Laplacian costing $\mc{O}(n^3)$. For small and medium-sized graphs, this is feasible to compute, but for large ones, it is not. We thus need a cheaper alternative for large graphs. When dealing with large graphs, motivated by our analysis of LPEs, instead of using all eigenvectors and eigenvalues, we propose to use the first and last $k$. These eigenvectors and eigenvalues can be obtained using the Arnoldi iteration algorithm \citep{arnoldi1951principle} with time complexity $\mc{O}(k^2n)$, assuming a sparse graph. Thus, for large sparse graphs, the eigenvectors and eigenvalues at the ends of the spectrum are feasible to obtain. We empirically demonstrate the algorithm's efficiency in Appendix \ref{sec:eff}. Furthermore, in the context of large graphs we demonstrate there is not a significant performance difference between LLPE and its approximate version for large graphs (Appendix \ref{app:neigs}).

\section{EXPERIMENTAL RESULTS \label{sec:exp}}

We evaluate LLPE on homophilous and heterophilous benchmarks, aiming to answer the following research questions:

\begin{itemize}
    \item \textbf{RQ1}: In a simple synthetic setting, where the data generation process and relevant eigenvectors are known, does LLPE and other PE baselines capture the relevant graph structure?
         
    \item \textbf{RQ2a}: On complex real-world small and medium graphs of varying homophily and heterophily, to what extent does LLPE and other PE baselines improve GNN performance? 

    \item \textbf{RQ2b}: On real-world large graphs, does the approximate version of LLPE scale and does it improve GNN performance?

\end{itemize}


\textbf{Implementation Details.} To obtain the eigenvectors and eigenvalues of the Laplacian, we utilize a fast implementation of Arnoldi iteration readily available in SciPy and optimized with sparse matrix and vector operations. We compute the Chebyshev polynomials using efficient tensor operations in Pytorch. Finally, LLPE is trained in an end to end fashion with the GNN via gradient descent.

\subsection{RQ1: Synthetic Data Experiments \label{sec:syn}}

\textbf{Data Generation.} We generate datasets according to binary and multiclass SBMs. The binary SBMs have $n=2000$ nodes while the multiclass SBMs have $n=5000$ nodes and $k=5$ communities. For both SBMs, we generate multiple graphs by varying the homophily ratio. Given an SBM, node labels are determined by the node's community, and node features are generated by sampling a vector of Gaussian distributed features. We present the remaining details in Appendix \ref{app:exp}\footnote{Code can be found at:\\https://github.com/MLD3/LearningLaplacianPEs}.

\textbf{Base Model and PE.} We test the performance of a GT (full) \citep{dwivedi2021generalization} with the PEs: 
(1)~No PE, 
(2)~LPE-FK which uses the first $k$ eigenvectors, 
(3)~LPE-FLK which uses the first and last $k$ eigenvectors, 
(4)~LPE-Full which leverages all the eigenvectors, 
(5)~ElasticPE \citep{liu2023graph}, and 
(6) our proposed encodings, LLPE. 
All PEs are concatenated to node features. 

\textbf{Training and Evaluation.} We train all model-PE combinations by minimizing the negative log-likelihood on the train set. We select the best performing model across hyperparameters on the validation set for evaluation on the test set. We report the mean $\pm$ standard deviation of the accuracy on the test set across 10 random splits ($60\%/20\%/20\%$), following \citet{Pei2020Geom-GCN}.

\textbf{Results.} In Figure \ref{fig:sbm}, we present results on the synthetic experiments. On homophilous SBMs, PEs that leverage the first $k$ eigenvectors perform well, while on heterophilous SBMs, PEs that leverage the last $k$ eigenvectors perform well. Since LLPE learns the relevant eigenvectors, it captures structure in both homophilous and heterophilous SBMs. We note that LPE-FLK performs identically to LLPE since the first and last $k$ eigenvectors are all that is necessary to capture the relevant structure on the SBMs. To test the approach on more complex structures, we conducted an additional experiment in Appendix \ref{sec:power} on synthetic graphs generated with preferential attachment \citep{zhu2020beyond}. On the more complex graphs, LPE-FLK is no longer able to capture the relevant structure, while LLPE is able to capture the relevant structure in these settings.

\begin{table*}[h!]
\begingroup
\setlength{\tabcolsep}{4pt} 
\small
\centering
\caption{Mean $\pm$ standard deviation of model-PE test accuracy across 10 random splits on the small benchmarks. We highlight in \textbf{\grn{green}} the best performing model-PE combination. We additionally count the number of test splits LLPE outperforms LPE-FK's performance indicated by $(\cdot/10)$.}
\label{tab:3}
\resizebox{\textwidth}{!}{
\begin{tabular}{l l c c c c c | c} 
\toprule
 & & \textbf{Texas} & \textbf{Cornell} & \textbf{Cora-ML} & \textbf{Cora} & \textbf{Photo} & \textbf{Avg.} \\ 
 \textbf{Model} & \textbf{PE/SE} & $\mathbf{h=0.00}$ & $\mathbf{h=0.15}$ & $\mathbf{h=0.74}$ & $\mathbf{h=0.75}$ & $\mathbf{h=0.76}$ & \textbf{Rank} \\
 \midrule
  & No PE & 82.64 ± 6.80 & 75.19 ± 5.84 & 73.05 ± 2.29 & 66.47 ± 2.60 & 90.19 ± 1.42 & 2.8 \\
  & LPE-FK & 78.93 ± 6.50 & 77.33 ± 5.23 & 71.57 ± 3.48 & 64.59 ± 3.58 & 89.43 ± 0.80 & 4.2 \\
  & LPE-FLK & 81.60 ± 8.96 & 74.13 ± 5.12 & 72.12 ± 3.85 & 63.93 ± 2.46 & 90.22 ± 0.73 & 4.0 \\
  MLP & LPE-Full & 79.22 ± 7.52 & 74.91 ± 5.90 & 70.45 ± 5.40 & 62.30 ± 2.08 & 88.77 ± 0.98 & 6.4\\
 & ElasticPE & 80.29 ± 6.01 & 75.16 ± 5.07 & 80.39 ± 3.08 & 63.26 ± 2.50 & 88.07 ± 0.98 & 4.8\\
 & RWSE & 81.05 ± 5.81 & 75.46 ± 7.10 & 70.58 ± 3.34 & 62.40 ± 2.59 & 89.22 ± 1.63 & 4.8\\
 & LLPE (ours) & \grn{\textbf{ 84.82 ± 6.05 (8/10)}} & \grn{\textbf{77.60 ± 4.61 (2/10)}} & \grn{\textbf{80.99 ± 1.67 (10/10)}} & \grn{\textbf{78.62 ± 1.68 (10/10)}} & \grn{\textbf{92.61 ± 0.74 (10/10)}} & \grn{\textbf{1.0}} \\
 \midrule
  & No PE & 80.29 ± 7.13 & 73.34 ± 3.51 & 87.92 ± 1.30 & 87.61 ± 1.20 & \grn{\textbf{95.32 ± 0.38}} & \grn{\textbf{2.8}} \\
  & LPE-FK & 80.29 ± 6.60 & 72.81 ± 3.24 & 87.49 ± 1.43 & 86.86 ± 0.82 & 95.12 ± 0.38 & 5.0 \\
  & LPE-FLK & 78.70 ± 7.96 & 73.08 ± 4.43 & 87.24 ± 1.50 & 86.84 ± 1.14 & 95.21 ± 0.51 & 5.8 \\
  SAGE & LPE-Full & 78.95 ± 6.63 & 74.42 ± 5.08 & 87.24 ± 1.82 & 87.08 ± 0.94 & 95.00 ± 0.42 & 5.2 \\
  & ElasticPE & 80.57 ± 7.74 & \grn{\textbf{77.59 ± 3.48}} & 87.64 ± 1.12 & 87.26 ± 0.91 & 95.04 ± 0.52 & 3.2 \\
  & RWSE & 81.36 ± 6.79 & 73.85 ± 4.34 & 87.62 ± 1.57 & 87.10 ± 0.70 & 95.21 ± 0.51 & 3.0 \\
  & LLPE (ours) & \grn{\textbf{83.99 ± 5.12 (6/10)}} & 72.28 ± 6.21 (4/10) & \grn{\textbf{88.17 ± 1.52 (7/10)}} & \grn{\textbf{88.28 ± 1.01 (10/10)}} & 95.12 ± 0.43 (5/10) & 3.0 \\
 \midrule
 & No PE & 84.52 ± 5.97 & 75.70 ± 7.25 & 78.52 ± 1.35 & 73.86 ± 2.05 & 91.82 ± 0.58 & 5.8 \\
 & LPE-FK & 84.80 ± 4.92 & 76.50 ± 8.85 & 78.73 ± 1.81 & 73.99 ± 2.32 & 92.00 ± 0.63 & 4.8 \\
 & LPE-FLK & \grn{\textbf{85.61 ± 5.44}} & 78.66 ± 5.61 & 78.00 ± 2.28 & 73.82 ± 2.28 & 91.61 ± 0.54 & 4.8 \\
 & LPE-Full & 85.33 ± 6.18 & 78.89 ± 6.43 & 77.66 ± 1.79 & 70.84 ± 2.18 & 91.53 ± 0.43 & 6.4 \\
 GT (full) & ElasticPE & 85.06 ± 4.81 & 78.11 ± 5.90 & \grn{\textbf{85.48 ± 1.15}} & 74.67 ± 1.68 & 91.70 ± 0.50 & 3.6 \\
 & SAN-PE & 82.15 ± 6.20 & 78.40 ± 5.77 & 78.40 ± 1.36 & 73.16 ± 1.40 & 91.19 ± 0.45 & 6.8 \\
 & SignNet & 84.82 ± 6.48 & 77.30 ± 6.57 & 78.15 ± 2.05 & 72.66 ± 2.46 & 91.70 ± 0.60 & 6.4\\
 & RWSE & 85.60 ± 9.16 & \grn{\textbf{79.18 ± 6.27}} & 78.10 ± 1.54 & 74.27 ± 2.28 & 91.55 ± 0.77 & 4.0 \\
 & LLPE (ours) & 85.34 ± 6.44 (4/10) & 78.39 ± 5.94 (6/10) & 84.50 ± 1.25 (10/10) & \grn{\textbf{80.83 ± 1.33 (10/10)}} & \grn{\textbf{94.34 ± 0.53 (10/10)}} & \grn{\textbf{2.4}} \\
 \bottomrule
 \label{tab:small}
\end{tabular}
}
\endgroup
\end{table*}

\begin{table*}[h!]
\begingroup
\setlength{\tabcolsep}{6pt} 
\small
\centering
\vspace{-1em}
\caption{Mean $\pm$ standard deviation of model-PE test accuracy (AUROC for Tolokers) across 10 random splits on the medium benchmarks. We follow the same conventions as in Table \ref{tab:small}.}
\label{tab:3}
\resizebox{\textwidth}{!}{
\begin{tabular}{l l c c c c | c} 
\toprule
 & & \textbf{Amazon-ratings} & \textbf{Tolokers} & \textbf{Cora-full} & \textbf{Computers} & \textbf{Avg.} \\ 
 \textbf{Model} & \textbf{PE/SE} & $\mathbf{h=0.12}$ & $\mathbf{h=0.17}$ & $\mathbf{h=0.50}$ & $\mathbf{h=0.70}$ & \textbf{Rank} \\
 \midrule
  & No PE & 43.84 ± 0.64 & 82.64 ± 0.79 & \grn{\textbf{68.68 ± 0.63}} & 91.02 ± 0.40 & 2.75 \\
  & LPE-FK & 43.88 ± 1.17 & 82.56 ± 0.55 & 67.91 ± 0.78 & 91.02 ± 0.44 & 4.00\\
  & ElasticPE & 42.84 ± 0.39 & 76.89 ± 2.52 & 66.87 ± 1.01 & 90.05 ± 0.53 & 6.50 \\
  SAGE & SAN-PE & 42.83 ± 0.78 & 80.37 ± 1.55 & 66.86 ± 0.61 & 90.97 ± 0.45 & 6.50 \\
  & SignNet & 43.48 ± 0.97 & 82.41 ± 0.54 & 68.24 ± 0.46 & 91.04 ± 0.44 & 3.75 \\
  & RWSE & 43.97 ± 1.21 & 82.43 ± 0.54 & 68.10 ± 0.73 & 91.06 ± 0.45 & 3.00 \\
  & LLPE (ours) & \grn{\textbf{45.56 ± 1.14 (10/10)}} & \grn{\textbf{83.46 ± 0.69 (10/10)}} & 68.12 ± 0.68 (5/10) & \grn{\textbf{91.08 ± 0.29 (6/10)}} & \grn{\textbf{1.50}}\\
 \midrule
 & No PE & 38.92 ± 0.60 & 74.04 ± 0.92 & 60.71 ± 0.70 & 85.13 ± 0.87 & 3.75 \\
 & LPE-FK & 38.59 ± 0.59 & 74.21 ± 0.63 & 59.33 ± 0.94 & 85.19 ± 0.74 & 4.75 \\
 & ElasticPE & 30.08 ± 5.48 & 73.23 ± 2.90 & 57.92 ± 1.39 & 85.28 ± 0.86 & 6.00 \\
 GT (full) & SAN-PE & 38.93 ± 0.60 & 78.42 ± 1.15 & 60.25 ± 0.60 & 85.36 ± 0.55 & 2.50 \\
 & SignNet & 38.61 ± 0.51 & 73.96 ± 0.86 & 60.28 ± 0.59 & 85.09 ± 0.68 & 5.00 \\
 & RWSE & 38.82 ± 0.52 & 74.09 ± 0.69 & 60.07 ± 0.87 & 85.05 ± 0.83 & 5.00 \\
 & LLPE (ours) & \grn{\textbf{39.80 ± 0.62 (10/10)}} & \grn{\textbf{80.85 ± 0.83 (10/10)}} & \grn{\textbf{61.02 ± 0.60 (9/10)}} & \grn{\textbf{87.83 ± 0.45 (10/10)}} & \grn{\textbf{1.00}} \\
 \bottomrule
 \label{tab:med}
 \vspace{-2em}
\end{tabular}
}
\endgroup
\end{table*}

\begin{table}[h!]
\begingroup
\setlength{\tabcolsep}{4pt} 
\small
\centering
\label{tab:3}
\caption{Mean $\pm$ standard deviation of test accuracy (AUROC for Questions) on the large benchmarks.}
\resizebox{.5\textwidth}{!}{

\begin{tabular}{l l c c } 
\toprule
 & & \textbf{Penn94} & \textbf{Questions} \\ 
 \textbf{Model} & \textbf{PE/SE} & $\mathbf{h=0.03}$ & $\mathbf{h=0.08}$ \\
 \midrule
  & No PE & 72.54 ± 0.51 & 74.77 ± 1.16 \\
  SAGE & LPE-FK & 72.71 ± 0.70 & 74.94 ± 1.18 \\
  & LLPE (large) & \grn{\textbf{72.79 ± 0.45 (5/10)}} & \grn{\textbf{77.84 ± 1.29 (10/10)}} \\
 \bottomrule
 \label{tab:large}
\end{tabular}
}
\vspace{-2em}
\endgroup
\end{table}

\subsection{RQ2a/b: Real-world Data Experiments}

Below, we describe our base models and PEs tested, training and evaluation procedure, and results for our real-world experiments. 

\textbf{Datasets.} We evaluate on 12 homophilous and heterophilous node classification datasets, dividing them into small benchmarks ranging from $300 - 8000$ nodes \citep{yang2016revisiting, bojchevski2018deep, shchur2018pitfalls, Pei2020Geom-GCN}, medium benchmarks ranging from $10000 - 25000$ nodes \citep{platonov2023critical}, and large benchmarks ranging from $40000 - 50000$ nodes \citep{lim2021large}. For each dataset we report the class homophily \citep{lim2021large}. 

\textbf{Base Models and PEs.} We test the following base models: 
(1)~MLPs, 
(2)~GTs~\citep{dwivedi2021generalization}, and 
(3)~GraphSage~\citep{hamilton2017inductive}. 
We then test the following PEs: 
(a)~No PE, 
(b)~LPE-FK \citep{dwivedi2021generalization}, 
(c)~LPE-FLK, 
(d)~LPE-Full, 
(e)~ElasticPE \citep{liu2023graph}, 
(f)~SignNet \citep{lim2022sign}, 
(g)~SAN-PE \citep{kreuzer2021rethinking}, 
(h)~RWSE \citep{dwivedi2022graph}, and 
(i)~our proposed encodings, LLPE. 

\begin{figure*}[h!]
     \centering
     \begin{subfigure}[b]{.48\textwidth}
         \centering
         \includegraphics[width=\textwidth]{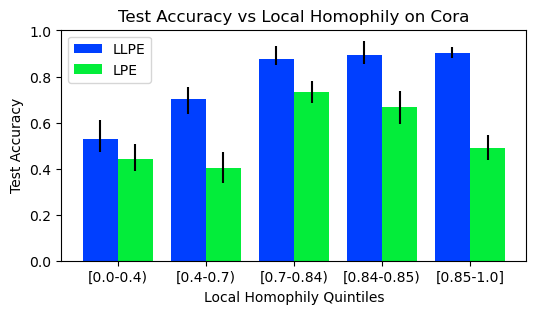}
         \caption{Accuracy for local homophily quintiles in Cora.}
     \end{subfigure}
     \hfill
     \begin{subfigure}[b]{.48\textwidth}
         \centering
         \includegraphics[width=\textwidth]{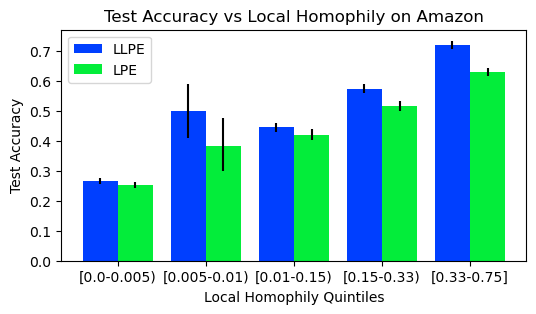}
         \caption{Accuracy for local homophily quintiles in Amazon.}
     \end{subfigure}
    \caption{Performance of GTs with LPE-FK and LLPE across local node homophily quintiles on Cora and Amazon-ratings. Error bars are based on 95\% bootstrapped confidence intervals.}
    \label{fig:analysis}
\end{figure*}

\textbf{Training and Evaluation.} We follow the same training and evaluation procedure in Section \ref{sec:syn}. Importantly, in order to make fair comparisons, we evaluate each PE across a wide range of $k$ eigenvectors and eigenvalues, letting $k \in [8, \text{all}]$ during hyperparameter selection, following \citet{kreuzer2021rethinking, lim2022sign}. Generally, we find that most PEs perform best under smaller choices of $k$ as opposed to larger ones. We present the remaining details in Appendix \ref{app:exp}. 

\textbf{Results.} Across small graph datasets, LLPE obtains the best average rank of 2.13, while the second best ranked PE obtains an average rank of 3.80 (Table \ref{tab:small}). Moreover, for 7/15 model and dataset combinations, LLPE outperforms LPE-FK 10/10 times across the test splits. In our medium datasets, we find similar results and LLPE obtains the best average rank of 1.25 in comparison to the second best rank of 3.25 (Table \ref{tab:med}). For 5/8 model-PE and dataset combinations, LLPE outperforms LPE-FK 10/10 times. In our large datasets, the approximate version of LLPE for large graphs outperforms LPE-FK and No-PE on both benchmarks (Table \ref{tab:large}). 

In our real-world experiments, we find that LPE-FLK does not capture the relevant structure, since real-world graphs require leveraging more intricate combinations of the eigenvectors. In many cases, learnable PEs such as SignNet and SAN-PE perform no better than LPE-FK such as in Amazon-ratings, Tolokers, and Cora, suggesting that existing learnable PEs may not identify the relevant eigenvectors in graphs of 1000s of nodes. In contrast, when given all eigenvectors where $n$ is as large as 25,000 LLPE does identify relevant eigenvectors as it consistently obtains higher performance compared to LPE-FK across many datasets. On large graphs of size 50,000, LLPE (large) is also able to identify relevant eigenvectors, improving performance over the two baselines.

Although LLPE obtains the best ranking across datasets, it exhibits limitations. In particular, LLPE requires tuning multiple hyperparameters, including the order $M$ of the Chebyshev polynomials and the number of eigenvectors $k$ on large datasets. We find that performance can degrade under certain hyperparameter selections. In particular, our approach is sensitive to small choices of $M$, leading to low approximation capabilities, and small choices of $k$ where too few eigenvectors are included in the PE (Appendix \ref{sec:add}).

\subsection{Sensitivity Analysis of LLPE \label{sec:loc}}

Interestingly, on small datasets LLPE's performance gains are larger for homophilous graphs. We investigate this result and find that LLPE's gains on homophilous graphs can be attributed to higher performance in local regions of high heterophily within the graph. In Figure \ref{fig:analysis}, we measure the performance of GT with LLPE and LPE-FK on Cora and Amazon across quintiles of local node homophily, homophily measured at the node level. Although Cora has high global homophily, it also contains a large portion of nodes that exhibit low local homophily (Figure \ref{fig:analysis}(a)). On these nodes, LLPE leads to large performance improvements in comparison to LPE-FK, thus demonstrating that LLPE can lead to performance improvements on graphs that are globally homophilous yet also contain local regions of heterophily. In Figure \ref{fig:analysis}(b), we analyze Amazon, a heterophilous benchmark. Here, we find that most nodes exhibit low local homophily, and LLPE leads to performance improvements across these nodes, resulting in an increase in overall performance on Amazon. 

\section{RELATED WORK}

In past work, researchers have analyzed the first $k$ eigenvectors of the Laplacian in the context of SBMs, focusing on regularization~\citep{le2015sparse, le2017concentration, guedon2016community}, while others have analyzed the largest $k$ eigenvectors of the adjacency matrix~\citep{rohe2011spectral, joseph2016impact}. Generally, spectral clustering is applied to these eigenvectors to extract communities from homophilous graphs. In contrast, we analyze the last $k$ eigenvectors of the Laplacian in Section \ref{sec:3}, demonstrating the connection between these eigenvectors and heterophily. We further propose a new PE leveraging the Laplacian's full spectrum rather than focusing on a subset.


Researchers have proposed many PEs that aim to improve LPEs with the addition of learnable components (see Appendix \ref{sec:learn_pes} for more details). However, all of these approaches focus on the first $k$ eigevenvectors since they assume homophilous graph structure. In contrast, we do not make such assumption and thus propose learning which eigenvectors are most relevant. While existing learnable PEs can be extended to the full spectrum of the Laplacian, we find that they do not capture relevant structure when provided the full spectrum since they are designed for small graphs containing 10-100 nodes and learning relevant eigenvectors is more challenging in large graphs of 1,000-50,000 nodes due to the high dimension of the eigenspace and increased noise. LLPE addresses this issue by leveraging learnable Chebyshev polynomials whose weights do not scale in $n$, resulting in better generalization.

Spectral GNNs \citep{defferrard2016convolutional, he2021bernnet, chien2021adaptive, wang2022powerful} also leverage polynomial processing of eigenvalues similar to LLPE. More specifically, given the eigendecomposition of the Laplacian, $\mbf{L} = \mbf{U}^\top \mbf{\Lambda} \mbf{U}$ and graph signal $\mbf{x}$, the spectral GNN operation is $ \mbf{U}^\top g(\mbf{\Lambda}) \mbf{U} \mbf{x}$, where $g(\mbf{\Lambda})$ are the processed eigenvalues that filter $\mbf{x}$ in the spectral domain enabled by the graph Fourier transform and its inverse, $\mbf{U}\mbf{x}$ and $\mbf{U}^\top\hat{\mbf{x}}$, respectively. In contrast, LLPE enhances GNN performance by directly leveraging Laplacian eigenvectors and eigenvalues as learnable positional encodings. It learns linear combinations of eigenvectors by defining $\mbf{UW}$, where $\mbf{W}$ is a matrix of learnable weights obtained by transformations on the eigenvalues. LLPE is then added to GNNs like message-passing neural networks or graph transformers as additional node positional information. Similar to how our work can improve heterophilous GNNs, our work can thus improve spectral GNNs by augmenting them with LLPE.


\section{CONCLUSION}

We present the first analysis of PEs on heterophilous benchmarks in node classification. Specifically, we demonstrate the limitations of popular PEs in capturing heterophily and propose a new PE, Learnable LPE. We demonstrate theoretically that LLPE captures homophily and heterophily by leveraging the full spectrum of the Laplacian. In our empirical analysis, we demonstrate that LLPE improves performance for a variety of GNNs across 12 node classification benchmarks. Our results indicate that many popular PEs do not capture heterophily, and thus going forward our work represents a significant step in developing data-driven PEs that capture complex structure on heterophilous graphs or heterophilous regions within the more prevalent class of homophilous graph datasets.

\subsubsection*{Acknowledgements}

This material is based upon work supported by the U.S. Department of Energy, Office of Science, Office of Advanced Scientific Computing Research, Department of Energy Computational Science Graduate Fellowship under Award Number DE-SC0023112. It was also partially supported by National Science Foundation under Grant No. IIS~2212143, and in part by an award from the Carl Friedrich von Siemens from the Alexander von Humboldt Foundation. We also thank the anonymous reviewers and members of the MLD3 lab for their valuable feedback.

\balance

\section*{Checklist}

 \begin{enumerate}

 \item For all models and algorithms presented, check if you include:
 \begin{enumerate}
   \item A clear description of the mathematical setting, assumptions, algorithm, and/or model. \textbf{Yes}
   \item An analysis of the properties and complexity (time, space, sample size) of any algorithm. \textbf{Yes}
   \item (Optional) Anonymized source code, with specification of all dependencies, including external libraries. \textbf{Yes}
 \end{enumerate}

 \item For any theoretical claim, check if you include:
 \begin{enumerate}
   \item Statements of the full set of assumptions of all theoretical results. \textbf{Yes}
   \item Complete proofs of all theoretical results. \textbf{Yes}
   \item Clear explanations of any assumptions. \textbf{Yes}
 \end{enumerate}

 \item For all figures and tables that present empirical results, check if you include:
 \begin{enumerate}
   \item The code, data, and instructions needed to reproduce the main experimental results (either in the supplemental material or as a URL). \textbf{No}
   \item All the training details (e.g., data splits, hyperparameters, how they were chosen). \textbf{Yes}
         \item A clear definition of the specific measure or statistics and error bars (e.g., with respect to the random seed after running experiments multiple times). \textbf{Yes}
         \item A description of the computing infrastructure used. (e.g., type of GPUs, internal cluster, or cloud provider). \textbf{Yes}
 \end{enumerate}

 \item If you are using existing assets (e.g., code, data, models) or curating/releasing new assets, check if you include:
 \begin{enumerate}
   \item Citations of the creator If your work uses existing assets. \textbf{Yes}
   \item The license information of the assets, if applicable. \textbf{Not Applicable}
   \item New assets either in the supplemental material or as a URL, if applicable. \textbf{Not Applicable}
   \item Information about consent from data providers/curators. \textbf{Not Applicable}
   \item Discussion of sensible content if applicable, e.g., personally identifiable information or offensive content. \textbf{Not Applicable}
 \end{enumerate}

 \item If you used crowdsourcing or conducted research with human subjects, check if you include:
 \begin{enumerate}
   \item The full text of instructions given to participants and screenshots. \textbf{Not Applicable}
   \item Descriptions of potential participant risks, with links to Institutional Review Board (IRB) approvals if applicable. \textbf{Not Applicable}
   \item The estimated hourly wage paid to participants and the total amount spent on participant compensation. \textbf{Not Applicable}
 \end{enumerate}

 \end{enumerate}

\appendix

\onecolumn


\aistatstitle{Supplementary Material}


\vspace{-5em}


\tableofcontents

\newpage

\section{PROOFS FOR LAPLACIAN POSITIONAL ENCODINGS}

\subsection{Background, Definitions, and Lemmas on Perturbation Theory and Graph Laplacians}

Perturbation theory analyzes how a function changes when its input is subject to perturbations. For our purposes, the functions of interest are the eigenvalues and eigenvectors of the graph Laplacian. Formally, our setup will focus on the expected graph Laplacian $\mbb{E}[\mbf{L}]$ according to a graph model and the observed graph Laplacian $\mbf{L} = \mbb{E}[\mbf{L}] + \mbf{E}$ that we observe from data. Intuitively, $\mbb{E}[\mbf{L}]$ is the informative component of $\mbf{L}$ since we can determine its eigenstructure while $\mbf{E}$ is the noise or error in $\mbf{L}$. Our main goal will be to quantify the distance between $\mbf{L}$'s eigenvalues and eigenvectors in comparison to $\mbb{E}[\mbf{L}]$'s eigenvalues and eigenvectors. 

When quantifying the distance between eigenvectors corresponding to simple eigenvalues, it is straightforward to compute distances, and we can use typical metrics defined on vectors. However, when generalizing to distances between eigenspaces it becomes much more challenging. In order to demonstrate this, consider the following example from \citet{stewart1990matrix}. We define matrices $\mbf{A}$ and its perturbations $\mbf{A}_1$ and $\mbf{A}_2$ for some $\epsilon > 0$, 
\eq{\mbf{A} = \begin{pmatrix}
1 & 0 & 0 \\
0 & 1 & 0 \\
0 & 0 & 0
\end{pmatrix}, \quad
\mbf{A}_1 = \begin{pmatrix}
1 & 0 & 0 \\
0 & 1 + \epsilon & 0 \\
0 & 0 & 0
\end{pmatrix}, \quad 
\mbf{A}_2 = \begin{pmatrix}
1 & \epsilon/2 & 0 \\
\epsilon/2 & 1 & 0 \\
0 & 0 & 0
\end{pmatrix}.}

First, notice that $\mbf{A}$ has no unique eigenvectors corresponding to its nonzero eigenvalue. Any vector lying in the span of the standard basis vectors $\mbf{e}_1$ and $\mbf{e}_2$ is an eigenvector of $\mbf{A}$. Thus, depending on the choice of eigenvectors for $\mbf{A}$, the perturbation bound could be large or small. Second, notice that $\mbf{A}_1$ has the same eigenvectors of $\mbf{A}$, while $\mbf{A}_2$'s eigenvectors, $(1, 1, 0)$ and $(1, -1, 0)$, are very different in comparison to $\mbf{A}$. Thus, depending on the nature of the perturbation, different eigenstructures may arise from $\mbf{A}$. Despite the differences from the two perturbations and the choice of eigenvectors for $\mbf{A}$ however, the eigenvectors of small perturbations of $\mbf{A}$ will always span a space close to the \textit{subspace} of the original eigenvectors. Since the subspace spanned by the eigenvectors is more stable than the eigenvectors themselves, we focus on bounding the distances between the subspaces. Specifically, we introduce notions of angles and distance on the set of $l$ dimensional subspaces of $\mbb{R}^n$. 

\begin{definition}
Let $\mc{X}, \mc{Y} \in \mbb{R}^n$ be two $l$ dimensional subspaces, and let the columns of $\mbf{X}$ and $\mbf{Y}$ form orthonormal bases for $\mc{X}$ and $\mc{Y}$. The unitarily invariant metric $\rho$ is defined, $\rho(\mc{X}, \mc{Y}) = \text{inf}_{\mbf{Q}\in O(l)} \lvert\lvert \mbf{X} - \mbf{YQ} \rvert\rvert_F$, where $O(l)$ denotes the set of all $l \times l$ orthogonal matrices and $\lvert\lvert \cdot \rvert\rvert_F$ denotes the Frobenius norm.
\label{def:1}
\end{definition}

We note two important properties of $\rho$. First, if $\mc{X} = \mc{Y}$, then there exists some unitary matrix $\mbf{Q}$ such that $\mbf{X} = \mbf{QY}$ and thus $\rho(\mc{X}, \mc{Y}) = 0$. The second is that $\rho$ is \textit{unitarily invariant}. Formally, for all $\mbf{U} \in O(l)$, $\rho(\mbf{U}\mc{X}, \mbf{U}\mc{Y}) = \rho(\mc{X}, \mc{Y})$. That is, rotations do not change the distance between $\mc{X}$ and $\mc{Y}$.

In our analysis, it is also essential to generalize the notion of angles between vectors to angles between subspaces. We use the following variational definition from \citet{rakocevic2003variational}.

\begin{definition}
Let $\mc{X}, \mc{Y} \in \mbb{R}^n$ be two $l$ dimensional subspaces, and let the columns of $\mbf{X}$ and $\mbf{Y}$ form orthonormal bases for $\mc{X}$ and $\mc{Y}$. Let $\mc{P}_{\mc{X}}$ and $\mc{P}_{\mc{Y}}$ be the orthogonal projectors onto $\mc{X}$ and $\mc{Y}$ and let $\frac{\pi}{2} \geq \sigma_0(\mc{P}_{\mc{X}} \mc{P}_{\mc{Y}}) \geq \cdots \geq \sigma_l(\mc{P}_{\mc{X}} \mc{P}_{\mc{Y}}) \geq 0$ be the singular values of $\mc{P}_{\mc{X}} \mc{P}_{\mc{Y}}$ ordered by decreasing magnitude. Then, the canonical angles between $\mc{X}$ and $\mc{Y}$ are the inverse cosine of the singular values, $\theta_k = \text{arccos}(\sigma_k(\mc{P}_{\mc{X}} \mc{P}_{\mc{Y}}))$.
\label{def:2}
\end{definition}

The following lemma elucidates the relationship between the metric $\rho$ and the canonical angles, 

\begin{lemma}[\citet{stewart1990matrix}]
Let $\mc{X}, \mc{Y} \in \mbb{R}^n$ be two $l$ dimensional subspaces, and let the columns of $\mbf{X}$ and $\mbf{Y}$ form orthonormal bases for $\mc{X}$ and $\mc{Y}$. Let $\rho$ be as defined in Definition \ref{def:1} and let $\Theta = \text{diag}(\theta_0, \ldots, \theta_l)$ where $\theta_i$ is defined in Definition \ref{def:2}. Then, $\rho(\mc{X}, \mc{Y}) \leq \sqrt{2}\nrm{\text{sin}\Theta}_F$
\label{lem:rho}
\end{lemma}

Now, we restate a variant of the Davis-Kahan theorem which bounds the canonical angles between the eigenspaces of two Hermitian matrices in terms of the distance between the two matrices \citep{yu2015useful}.

\begin{theorem}[\citet{yu2015useful}, $\text{sin}\Theta$ theorem]
Let $\mbf{X}$ and $\tilde{\mbf{X}} \in \mbb{R}^{n \times n}$ be Hermitian matrices with eigenvalues $\lambda_1, \ldots, \lambda_n$ and $\tilde \lambda_1, \ldots, \tilde \lambda_n$ respectively. Fix $1 \leq r \leq s \leq p$, let $d = s - r + 1$ and $\mbf{V} = (\mbf{v}_r, \mbf{v}_{r+1}, \ldots, \mbf{v}_s) \in \mbb{R}^{n \times d}$ and let $\tilde{\mbf{V}} = (\tilde{\mbf{v}}_r, \tilde{\mbf{v}}_{r+1}, \ldots, \tilde{\mbf{v}}_s) \in \mbb{R}^{n \times d}$ have orthonormal columns satisfying $\mbf{X}\mbf{v}_j = \lambda_j\mbf{v}_j$ and $\tilde{\mbf{X}} \tilde{\mbf{v}}_j = \tilde \lambda_j \tilde{\mbf{v}}_j$ for $j = r, r+1, \ldots, s$. Let $\Theta$ be the diagonal matrix of canonical angles between $\mbf{V}$ and $\tilde{\mbf{V}}$. Assume $d \ll n$, then, 
\eq{\nrm{\text{sin}\Theta}_F \leq \frac{2d^{1/2}\nrm{\mbf{X} - \tilde{ \mbf{X}}}}{\text{min}\{\lambda_{r-1} - \lambda_r, \lambda_s - \lambda_{s-1}\}}}
\label{thm:sin}
\end{theorem}


A key result that we leverage in our analysis is the concentration of the Laplacian \citep{oliveira2009concentration}. Specifically, let $\mbb{E}[\mbf{L}]$ and $\mbf{L}$ be the expected and observed graph Laplacians of graph $G$, respectively. We say the Laplacian concentrates about its expectation if $\mbb{E}[\mbf{L}]$ is close to $\mbf{L}$ as measured by the operator norm $\nrm{\mbb{E}[\mbf{L}] - \mbf{L}}$. Since $\nrm{\cdot}$ is the operator norm, concentration results of this nature imply a tight control of the eigenvectors of the Laplacian up to perturbations as given by the Davis-Kahan theorem. We will specifically be interested in stochastic block model graphs as defined by $G(n, p, q, k)$, where $n$ is the number of nodes, $k$ is the number of communities, $p$ is the probability of two nodes from the same community forming an edge, and $q$ is the probability of two nodes from different communities forming an edge. For simplicity, we assume the dense regime where we let the minimum node degree be $\text{min}(d_i) = C\text{ln}(n)$ for some constant $C$. This assumption is necessary for the concentration of the \textit{unregularized} Laplacian since $\text{min}(d_i) < C\text{ln}(n)$ implies $\nrm{\mbb{E}[\mbf{L}] - \mbf{L}} \geq 1$. Note that concentration can be extended to sparse Laplacians with proper regularization as shown in \citep{le2015sparse, le2017concentration}. Below, we formally restate the concentration of the graph Laplacian in the dense regime, 

\begin{theorem}[\citet{oliveira2009concentration}, Concentration of the Laplacian]
Let $G(n, p, q, k)$ be a stochastic block model, $\mbb{E}[\mbf{L}]$ the expected Laplacian of $G$, and $\mbf{L}$ the observed Laplacian of $G$. Let $C \geq 0$ be a constant independent of $n, p, q$ and $k$. If $\text{min}(d_i) \geq C \text{ln}(n)$, then for $n^{-c} \leq \delta \leq 1/2$, 
\eq{\mbb{P}\prn{\nrm{\mbb{E}[\mbf{L}] - \mbf{L}} \leq 14 \sqrt{\frac{\text{ln}(4n/\delta)}{\text{min}(d_i)}}} \geq 1 - \delta}
\label{thm:con}
\end{theorem}

\subsection{Laplacian Positional Encodings \label{pf:lap}}

With the tools introduced on perturbation theory and the concentration results for the Laplacian in the dense regime, we can prove our results for Laplacian positional encodings. We begin proving our results for heterophilous block models and show that our results and derivations can be extended to homophilous block models in a straightforward fashion. We begin with results on binary SBMs and later generalize our results to multiclass SBMs as in the main paper. 

\begin{proposition}
Let $\mbf{A}$ and $\mbf{L}$ be the Adjacency and Laplacian matrix drawn from the stochastic block model $G(n, 2, p, q)$. Assume $q \gg p$ and $\text{min}(d_i) \geq C \text{ln}(n)$ where $C$ is an appropriately large constant. Then, with high probability, the signs of the entries of the \textbf{last eigenvector} of $\mbf{L}$ correctly recover the true communities up $\mc{O}(1)$ of misclassified nodes. Moreover, the first nontrivial eigenvector \textbf{does not} recover the true communities.
\label{prp:2}
\end{proposition}

\begin{proof}
We first decompose $\mbf{L}$ into its signal and error components, $\mbf{L} = \mbb{E}[\mbf{L}] + \mbf{E}$, where $\mbb{E}[\mbf{L}]$ is the expected graph Laplacian and $\mbf{E}$ is the remaining residual term. Since $\mbf{L}$ is drawn from a stochastic block model with parameters $n, 2, p, q$, we can exactly solve for the characteristic polynomial, eigenvalues, and eigenvectors of its expectation. We express $\mbb{E}[\mbf{L}]$ as, 
\eq{\mbb{E}[\mbf{L}] = I - \mbb{E}[\mbf{D}]^{1/2} \mbb{E}[\mbf{A}] \mbb{E}[\mbf{D}]^{1/2} \label{eq:7}}

where $\mbb{E}[\mbf{A}]$ is the expected Adjacency and $\mbb{E}[\mbf{D}]$ is the expected diagonal degree matrix. If we order all nodes according to the community they belong to, the components in Eq. \ref{eq:7} can be written as, 

\eq{\mbb{E}[\mbf{A}] = \begin{pmatrix}
\mbf{P_A} & \mbf{Q_A} \\
\mbf{Q_A} & \mbf{P_A} \\
\end{pmatrix}, \quad \mbb{E}[\mbf{D}] = \begin{pmatrix}
\frac{pn}{2} + \frac{qn}{2} & \cdots & 0\\
\vdots & \ddots & \vdots \\
0 & \cdots & \frac{pn}{2} + \frac{qn}{2} \\
\end{pmatrix}}

where $\mbb{E}[\mbf{A}]$ is written in block form with $\mbf{P_A}, \mbf{Q_A} \in \mbb{R}^{\frac{n}{2} \times \frac{n}{2}}$ where $\mbf{P_A} = \mbf{1}_{\frac{n}{2}}\mbf{1}_{\frac{n}{2}}^\top \cdot p$ and $\mbf{Q_A}  = \mbf{1}_{\frac{n}{2}}\mbf{1}_{\frac{n}{2}}^\top \cdot q$ where $\mbf{1}_{\frac{n}{2}} \in \mbb{R}^{\frac{n}{2}}$ is a $\frac{n}{2}$-dimensional vector of all ones and $\mbb{E}[\mbf{D}]$ is written as a diagonal matrix with $\mbb{E}[\mbf{D}] = I \cdot (\frac{pn}{2} + \frac{qn}{2})$. Now, $\mbb{E}[\mbf{L}]$ can be written, 
\eq{\mbb{E}[\mbf{L}] = \begin{pmatrix}
\mbf{P_L} & \mbf{Q_L} \\
\mbf{Q_L} & \mbf{P_L} \\
\end{pmatrix},}

where again we write $\mbb{E}[\mbf{L}]$ in block form where $\mbf{P_L}, \mbf{Q_L} \in \mbb{R}^{\frac{n}{2} \times \frac{n}{2}}$ and are defined, 
\eq{\mbf{P_L} =
\begin{pmatrix}
1 - \frac{p}{\frac{n}{2}(p+q)} & -\frac{p}{\frac{n}{2}(p+q)} & \cdots & -\frac{p}{\frac{n}{2}(p+q)}\\
-\frac{p}{\frac{n}{2}(p+q)} & 1 - \frac{p}{\frac{n}{2}(p+q)} & \cdots & -\frac{p}{\frac{n}{2}(p+q)}\\
\vdots & \vdots & \ddots & \vdots \\
-\frac{p}{\frac{n}{2}(p+q)} & -\frac{p}{\frac{n}{2}(p+q)} & \cdots & 1 - \frac{p}{\frac{n}{2}(p+q)}\\
\end{pmatrix}, \quad \mbf{Q_L} = 
\begin{pmatrix}
-\frac{q}{\frac{n}{2}(p+q)} & \cdots & -\frac{q}{\frac{n}{2}(p+q)}\\
\vdots & \ddots & \vdots \\
-\frac{q}{\frac{n}{2}(p+q)} & \cdots & -\frac{q}{\frac{n}{2}(p+q)} \\
\end{pmatrix}.}

Having derived the exact form of $\mbb{E}[\mbf{L}]$, we can derive its characteristic polynomial $p_{\mbb{E}[\mbf{L}]}(\lambda)$, 
\eq{p_{\mbb{E}[\mbf{L}]}(\lambda) = \frac{\lambda (\lambda - 1)^{n-2} (p\lambda + q(\lambda - 2))}{p+q}. \label{eq:12}}

Equation \ref{eq:12} tells us that the eigenvalues of $\mbb{E}[\mbf{L}]$ are $\lambda = 0, 1, \frac{2q}{p + q}$ with multiplicities $1, n - 2, 1$. By assumption, $q \gg p$, implying $\frac{2q}{p + q} > 1$, where the last eigenvector corresponds to the eigenvalue $\frac{2q}{p + q}$. The eigenvector of $\frac{2q}{p + q} $ can then be written as the vector $\mbf{v} \in \mbb{R}^{n \times 1}$, 
\eq{\mbf{v}^\top = \begin{pmatrix}
-1 & \cdots & -1 & +1 & \cdots & +1 \\
\end{pmatrix}}

where the first $\frac{n}{2}$ entries of $\mbf{v}$ are $-1$ and the last $\frac{n}{2}$ entries of $\mbf{v}$ are $+1$. Let $\mbf{\tilde v}$ denote the last eigenvector of $\mbf{L}$. Combining Theorem \ref{thm:sin}, Theorem \ref{thm:con}, and Lemma \ref{lem:rho} we have with probability $1 - \delta$, 
\al{\text{inf}_{q\in \{-1, +1\}} \lvert\lvert \mbf{v} - q \mbf{\tilde v} \rvert\rvert_2 &\leq \frac{2\nrm{\mbb{E}[\mbf{L}] - \mbf{L}}}{\text{min}\{\lambda_{r-1} - \lambda_r, \lambda_s - \lambda_{s-1}\} } \\
&\leq \frac{28 \sqrt{\frac{\text{ln}(4n/\delta)}{d_\text{min}}}}{\frac{2q}{p + q} -1} \\
&= \frac{28 \sqrt{\frac{\text{ln}(4n/\delta)}{d_\text{min}}} (p + q)}{q-p} \\
&= \mc{O}(1)
}

Thus, with high probability the last eigenvector of the observed Laplacian $\mbf{L}$ recovers the communities up to the sign of its entries with $\mc{O}(1)$ misclassified nodes. In order to prove our second claim, we note eigenvalue $\lambda=1$ of $\mbb{E}[\mbf{L}]$ has a multiplicty of $n - 2$ and its eigenvectors have the form, 
\eq{\mbf{\tilde v}^\top \in \{ \begin{pmatrix} 
-1 & +1 & 0 & \cdots & 0 & 0 & \cdots & 0 \\
\end{pmatrix}, \text{ } \cdots, \begin{pmatrix} 
0 & \cdots & 0 & -1 & 0 & \cdots & 0 & +1 \\
\end{pmatrix}\},}

where $\mbf{v}$ is an all zeros vector other than one index set to $-1$ and another other set to 1. Assuming we have ordered the nodes according to their community, the index with entry $-1$ is the index of the first node in either the first or second community while the index with entry 1 is any other index of a node in that community. Since $\mbf{\tilde v}$ is a vector of all zeros other than two other entries, any $k - 1$ of them cannot recover the true communities for all $n$ nodes. 
\end{proof}

\begin{theorem}
Let $\mbf{A}$ and $\mbf{L}$ be the Adjacency and Laplacian matrix drawn from the stochastic block model $G(n, k, p, q)$. Assume $q \gg p$ and $d_\text{min} \geq C \text{ln}(n)$ where $C$ is an appropriately large constant. Then, with high probability, the nonzero entries along the rows of the \textbf{last $k-1$ eigenvectors} of $\mbf{L}$ correctly recover the true communities up to an orthogonal transformation with at most $\mc{O}(k^\frac{3}{2})$ misclassified nodes. Moreover, the first nontrivial $k$ eigenvectors \textbf{do not} recover the true communities.
\end{theorem}

\begin{proof}
We first decompose $\mbf{L}$ into its signal and error components, $\mbf{L} = \mbb{E}[\mbf{L}] + E$, where $\mbb{E}[\mbf{L}]$ is the expected graph Laplacian and $E$ is the remaining residual term. Since $\mbf{L}$ is drawn from a stochastic block model with parameters $n, k, p, q$, we can exactly solve for the characteristic polynomial, eigenvalues, and eigenvectors of its expectation. We express $\mbb{E}[\mbf{L}]$ as, 

\eq{\mbb{E}[\mbf{L}] = I - \mbb{E}[\mbf{D}]^{1/2} \mbb{E}[\mbf{A}] \mbb{E}[\mbf{D}]^{1/2} \label{eq:7}}

where $\mbb{E}[\mbf{A}]$ is the expected Adjacency and $\mbb{E}[\mbf{D}]$ is the expected diagonal degree matrix. If we order all nodes according to the community they belong to, the components in Eq. \ref{eq:7} can be written as, 

\eq{\mbb{E}[\mbf{A}] = \begin{pmatrix}
\mbf{P_A} & \mbf{Q_A} & \cdots & \mbf{Q_A}\\
\mbf{Q_A} & \mbf{P_A} & \cdots & \mbf{Q_A} \\
\vdots & \vdots & \ddots & \vdots \\
\mbf{Q_A} & \mbf{Q_A} & \cdots & \mbf{P_A} \\
\end{pmatrix}, \quad \mbb{E}[\mbf{D}] = \begin{pmatrix}
\frac{pn}{k} + q(n - \frac{n}{k}) & \cdots & 0\\
\vdots & \ddots & \vdots \\
0 & \cdots & \frac{pn}{k} + q(n - \frac{n}{k}) \\
\end{pmatrix}}

where $\mbb{E}[\mbf{A}]$ is written in block form with $\mbf{P_A}, \mbf{Q_A} \in \mbb{R}^{\frac{n}{k} \times \frac{n}{k}}$ where $\mbf{P_A} = \mbf{1}_k\mbf{1}_k^\top \cdot p$ and $\mbf{Q_A}  = \mbf{1}_k\mbf{1}_k^\top \cdot q$ where $\mbf{1}_k \in \mbb{R}^{k}$ is a $k$-dimensional vector of all ones and $\mbb{E}[\mbf{D}]$ is written as a diagonal matrix with $\mbb{E}[\mbf{D}] = I \cdot (\frac{pn}{k} + q(n - \frac{n}{k}))$. Now, $\mbb{E}[\mbf{L}]$ can be written, 
\eq{\mbb{E}[\mbf{L}] = \begin{pmatrix}
\mbf{P_L} & \mbf{Q_L} & \cdots & \mbf{Q_L}\\
\mbf{Q_L} & \mbf{P_L} & \cdots & \mbf{Q_L} \\
\vdots & \vdots & \ddots & \vdots \\
\mbf{Q_L} & \mbf{Q_L} & \cdots & \mbf{P_L} \\
\end{pmatrix},}

where again we write $\mbb{E}[\mbf{L}]$ in block form where $\mbf{P_L}, \mbf{Q_L} \in \mbb{R}^{\frac{n}{k} \times \frac{n}{k}}$ and are defined, 
\eq{\mbf{P_L} =
\begin{pmatrix}
1 - \frac{p}{\frac{np}{k} + q(n - \frac{n}{k})} & -\frac{p}{\frac{np}{k} + q(n - \frac{n}{k})} & \cdots & -\frac{p}{\frac{np}{k} + q(n - \frac{n}{k})}\\
-\frac{p}{\frac{np}{k} + q(n - \frac{n}{k})} & 1 - \frac{p}{\frac{np}{k} + q(n - \frac{n}{k})} & \cdots & -\frac{p}{\frac{np}{k} + q(n - \frac{n}{k})}\\
\vdots & \vdots & \ddots & \vdots \\
-\frac{p}{\frac{np}{k} + q(n - \frac{n}{k})} & -\frac{p}{\frac{np}{k} + q(n - \frac{n}{k})} & \cdots & 1 - \frac{p}{\frac{np}{k} + q(n - \frac{n}{k})}\\
\end{pmatrix} \text{, and }}

\eq{\mbf{Q_L} = 
\begin{pmatrix}
-\frac{q}{\frac{np}{k} + q(n - \frac{n}{k})} & \cdots & -\frac{q}{\frac{np}{k} + q(n - \frac{n}{k})}\\
\vdots & \ddots & \vdots \\
-\frac{q}{\frac{np}{k} + q(n - \frac{n}{k})} & \cdots & -\frac{q}{\frac{np}{k} + q(n - \frac{n}{k})} \\
\end{pmatrix}.}

Having derived the exact form of $\mbb{E}[\mbf{L}]$, we can derive its characteristic polynomial $p_{\mbb{E}[\mbf{L}]}(\lambda)$, 
\eq{p_{\mbb{E}[\mbf{L}]}(\lambda) = \frac{\lambda (\lambda - 1)^{n-k} (p\lambda + q((k-1)\lambda - k))^{k-1}}{(p + (k-1)q)^{k-1}}. \label{eq:27}}
Eq. \ref{eq:27} tells us that the eigenvalues of $\mbb{E}[\mbf{L}]$ are $\lambda = 0, 1, \frac{kq}{p + (k-1)q}$ with multiplicities $1, n - k, k-1$. By assumption, $q > p$, implying $\frac{kq}{p + (k-1)q} > 1$, where the last $k - 1$ eigenvectors correspond to the eigenvalue $\frac{kq}{p + (k-1)q}$. The eigenvectors of $\frac{kq}{p + (k-1)q}$ can then be written as the columns of the matrix $\mbf{V} \in \mbb{R}^{n \times k - 1}$, 
\eq{\mbf{V} = \begin{pmatrix}
-\mbf{1}_k & -\mbf{1}_k & \cdots & -\mbf{1}_k\\
\mbf{1}_k & \mbf{0}_k & \cdots & \mbf{0}_k \\
\mbf{0}_k & \mbf{1}_k & \cdots & \mbf{0}_k \\
\vdots & \vdots & \cdots & \vdots \\
\mbf{0}_k & \mbf{0}_k & \cdots & \mbf{1}_k \\
\end{pmatrix}}

where $\mbf{0}_k, \mbf{1}_k \in \mbb{R}^{k}$ are $k$-dimensional vector of all zeros and ones, respectively. Let $\tilde{\mbf{V}}$ denote the last $k - 1$ eigenvectors of $\mbf{L}$. Combining Theorem \ref{thm:sin}, Theorem \ref{thm:con}, and Lemma \ref{lem:rho} we have with probability $1 - \delta$, 
\al{\text{inf}_{\mbf{Q}\in O(l)} \lvert\lvert \mbf{V} - \tilde{\mbf{V}}\mbf{Q} \rvert\rvert_F &\leq \frac{2k^{1/2}\nrm{\mbb{E}[\mbf{L}] - \mbf{L}}}{\text{min}\{\lambda_{r-1} - \lambda_r, \lambda_s - \lambda_{s-1}\} } \\
&\leq \frac{28 k^{1/2}\sqrt{\frac{\text{ln}(4n/\delta)}{d_\text{min}}}}{\frac{kq}{p + (k-1)q} - 1} \\
&= \frac{28 k^{1/2}\sqrt{\frac{\text{ln}(4n/\delta)}{d_\text{min}}} (p + q(k-1))}{q-p} \\
&= \mc{O}(k^\frac{3}{2})
}

Thus, with high probability the last $k - 1$ eigenvectors of the observed Laplacian $\mbf{L}$ recover the communities up to the nonzero entries of its rows with at most $\mc{O}(k^{\frac{3}{2}})$ misclassified nodes. In order to prove our second claim, we note eigenvalue $\lambda=1$ of $\mbb{E}[\mbf{L}]$ has a multiplicty of $n - k$ and its eigenvectors have the form, 
\eq{\mbf{\tilde v}^\top \in \{ \begin{pmatrix} 
-1 & +1 & 0 & \cdots & 0 & 0 & \cdots & 0 \\
\end{pmatrix}, \text{ } \cdots, \begin{pmatrix} 
0 & \cdots & 0 & -1 & 0 & \cdots & 0 & +1 \\
\end{pmatrix}\},}

where $\mbf{v}$ is an all zeros vector other than one index set to $-1$ and another other set to 1. Assuming we have ordered the nodes according to their community, the index with entry $-1$ is the index of the first node in some community $k$ while the index with entry 1 is any other index of a node in community $k$. Since $\mbf{v}$ is a vector of all zeros other than two other entries, any $k - 1$ of them cannot recover the true communities for all $n$ nodes. 
\end{proof}

Our results on the heterophilous SBMs can be extended to homophilous SBMs since the derivation for the expected Laplacian, characteristic polynomial, eigenvalues, and eigenvectors remain the same. In fact, the only difference is with the assumption $p \gg q$, and here, the eigenvector(s) of interest lie at the start of the spectrum rather than at the end. We demonstrate this phenomenon below.

\begin{proposition}[\citet{le2015sparse, le2017concentration}]
Let $\mbf{A}$ and $\mbf{L}$ be the Adjacency and Laplacian matrix drawn from the stochastic block model $G(n, 2, p, q)$. Assume $p \gg q$ and $\text{min}(d_i) \geq C \text{ln}(n)$ where $C$ is an appropriately large constant. Then, with high probability, the signs of the entries of the \textbf{first nontrivial eigenvector} of $\mbf{L}$ correctly recovers the true communities up $\mc{O}(1)$ of misclassified nodes.
\end{proposition}

\begin{proof}
Since in the derivation for $\mbb{E}[\mbf{L}]$ and $p_{\mbb{E}[\mbf{L}]}(\lambda)$ in Proposition \ref{prp:2} we did not rely on the assumption that $q \gg p$ and we are assuming a block model defined as $G(n, 2, p, q)$, $\mbb{E}[\mbf{L}]$ and $p_{\mbb{E}[\mbf{L}]}(\lambda)$ are the same as in Proposition \ref{prp:2}. Again, the eigenvalues of $\mbb{E}[\mbf{L}]$ are $\lambda = 0, 1, \frac{2q}{p+q}$. Now, by assumption $p \gg q$, implying $\frac{2q}{p + q} < 1$, and as a result the \textit{first nontrivial eigenvector} corresponds to eigenvalue $\frac{2q}{p+q}$. The proof follows similarly as in Proposition \ref{prp:2} except here we apply Theorems \ref{thm:sin}, Theorem \ref{thm:con}, and Lemma \ref{lem:rho} to the first nontrivial eigenvector rather than the last eigenvector, and we arrive at our conclusion.
\end{proof}

\begin{theorem}
Let $\mbf{A}$ and $\mbf{L}$ be the Adjacency and Laplacian matrix drawn from the stochastic block model $G(n, k, p, q)$. Assume $p \gg q$ and $\text{min}(d_i) \geq C \text{ln}(n)$ where $C$ is an appropriately large constant. Then, with high probability, the nonzero entries along the rows of the \textbf{first nontrivial $k-1$ eigenvectors} of $\mbf{L}$ correctly recovers the true communities up to an orthogonal transformation with at most $\mc{O}(k^\frac{3}{2})$ misclassified nodes.
\end{theorem}

\begin{proof}
Since in the derivation for $\mbb{E}[\mbf{L}]$ and $p_{\mbb{E}[\mbf{L}]}(\lambda)$ in Theorem \ref{thm:2} we did not rely on the assumption that $q \gg p$ and we are assuming a block model defined as $G(n, k, p, q)$, $\mbb{E}[\mbf{L}]$ and $p_{\mbb{E}[\mbf{L}]}(\lambda)$ are the same as in Theorem \ref{thm:2}. Again, the eigenvalues of $\mbb{E}[\mbf{L}]$ are $\lambda = 0, 1, \frac{kq}{p+(k-1)q}$. Now, by assumption $p \gg q$, implying $\frac{kq}{p+(k-1)q} < 1$, and as a result the \textit{first nontrivial $k-1$ eigenvectors} correspond to eigenvalue $\frac{kq}{p+(k-1)q}$. The proof follows similarly as in Theorem \ref{thm:2} except here we apply Theorems \ref{thm:sin}, Theorem \ref{thm:con}, and Lemma \ref{lem:rho} to the first nontrivial $k-1$ eigenvectors rather than the last $k-1$ eigenvectors, and we arrive at our conclusion.
\end{proof}

\section{PROOFS FOR LEARNABLE LAPLACIAN POSITIONAL ENCODINGS}

\subsection{Background, Definitions, and Lemmas on Chebyshev Polynomials}

Here we provide the background on approximation theory, Chebyshev Polynomials, and the approximation power of Chebyshev series necessary to prove our theorems for LLPE. We first introduce relevant notions from approximation theory, beginning with the $l^\infty$ norm, allowing us to measure the closeness of an arbitrary $f$ to $g$. We next review the definition of the Chebyshev polynomial, Chebyshev series, and the approximation power of truncated Chebyshev series. Further discussions on Chebyshev polynomials can be found in \citet{rivlin2020chebyshev, mason2002chebyshev}. 
\begin{definition} 
The $l^\infty$ norm of function $f$ on interval $[a, b]$ is defined, 
\eq{\nrm{f}_\infty = \text{max}_{a \leq x \leq b} \abs{f(x)}}
\end{definition}

If for some prescribed $\epsilon$, $\nrm{f - g} \leq \epsilon$, we say that $g$ uniformly approximates $f$. Moreover, if for some $g^*$, $\nrm{f - g^*} \leq \nrm{f - g}$ for all $g$, we say that $g^*$ is a best approximation to $f$. 

The Chebyshev polynomial $T_m(x)$ of the first kind is a polynomial of degree $m$ defined as, 
\eq{T_m(x) = \text{cos}(m \cdot \text{arccos}(x))}

when $0 \leq \text{arccos}(x) \leq \pi$. One crucial property of the Chebyshev polynomial of degree $m$ is that $\tilde T_m$ the normalized Chebyshev polynomial such that its leading coefficient is 1 has the minimum $l^\infty$ norm among all polynomials of degree $m$. The following theorem states this formally,
\begin{theorem}[\citet{rivlin2020chebyshev}]
Let $p_m \in \ms{P}_m$ be a polynomial of degree $m$. Then, 
\eq{\nrm{p_m}_\infty \geq \nrm{\tilde T_m}_\infty = 
\begin{cases}
2^{1-m},\quad m > 0\\
1,\quad m=0\\
\end{cases}}
\label{thm:nrm}
\end{theorem}

We later use Theorem \ref{thm:nrm} to prove optimal statistical generalization for LLPE among all other choices of approximating polynomials. Now, for any function $f$, there exists its Chebyshev series denoted as, 
\eq{f(x) \sim \sum_{m=0}^\infty a_m T_m(x),\quad \text{where } a_k = \frac{2}{\pi} \int_{-1}^{1} f(x) T_k(x) \, \frac{\partial x}{\sqrt{1 - x^2}}}

If $f$ is continuous, its Chebyshev series is pointwise convergent to $f$. Moreover, we can obtain uniform convergence if we place stricter assumptions on $f$. In particular, if $f$ satisfies the Dini-Lipschitz condition such that 
\eq{\text{lim}_{m\to\infty} \text{log}(m)\omega\prn{\frac{1}{m}} = 0,}

where $\omega(\delta) = \text{sup}_{x_1 - x_2 \leq \delta} \abs{f(x_1) - f(x_2)}$, then its Chebyshev series is uniformly convergent. The above discussion tell us that the Chebyshev series approximates functions that satisfy certain continuity and Lipschitz constraints. However, in practice, we are limited to partial sums of the Chebyshev series, and thus we will be mainly concerned with the approximating power of the partial sums of degree $M$. We denote the $M$th partial sum of the Chebyshev series of $f$ as, 
\eq{s_M(f; x) = s_M(x) = \sum_{m=0}^M a_m T_m(x).}

First we will discuss the relationship between the error of the best approximating polynomial of degree $M$ to $f$ and the error of the $M$th partial sum of the Chebyshev series to $f$. 

\begin{theorem}[\citet{rivlin2020chebyshev}]
Let $f$ be continuous and define, 
\eq{S_M(f) = \nrm{f - s_M(f)}, \quad E_M(f) = \nrm{f - p_M^*},}

where $p_M^*$ is the best uniform approximating polynomial of degree $M$ to $f$. Then,
\eq{E_M(f) \leq S_M(f) < \prn{4 + \frac{4}{\pi^2} \text{log}(n)} E_M(f)}
\label{thm:rel}
\end{theorem}

Theorem \ref{thm:rel} tells us that the error of $s_M(f)$ doesn't deviate by \textit{too much} from the error of the best approximating polynomial of degree $M$, $p_M^*$. More specifically, the error grows on the order of $\mc{O}(\text{log}(M))$, and as result, the $M$th partial sum of the Chebyshev series is a \textit{near-best} approximation of the function $f$. The following theorem will be our main tool used in proving the approximation theorems for LLPE and states that partial sums of the Chebyshev series converge exponentially if $f$ is continuously differentiable, 
\begin{theorem}[\citet{mason2002chebyshev}]
If $f$ has $d+1$ continuous derivatives on $[-1, +1]$, then $\nrm{f(x) - s_M(f)} = O(n^{-d})$.
\label{thm:chb}
\end{theorem}

\subsection{Theoretical Expressivity of LLPE \label{pf:llpe_app}}

\begin{definition}
Let $G$ be an arbitrary graph. Define $f_r: V \times V \to \mbb{R}$ as a function on $G$ with respect to $r: [0, 2] \to \mbb{R^+}$ such that $f_r(i, j)$ has the following form,
\eq{f_r(i, j)^2 = \sum_{k=1}^n r(\lambda_k) (\mbf{u}_k[i] - \mbf{u}_k[j])^2.}
\end{definition}

\begin{theorem}
Let $G$ be an arbitrary graph and $d_r$ be a function on $G$ of the form in Definition \ref{def:dist} for some $r$. LLPE can recover $d_r$ such that for any nodes $i$ and $j$, the $l^2$ distance between LLPE's encoding for nodes $i$ and $j$ approximates $d_r(i, j)$.
\end{theorem}

\begin{proof}
Let $\epsilon > 0$. Formally, we aim to show, 
\eq{ \Big| \nrm{\mbf{P}_\text{LLPE}[i, : ] - \mbf{P}_\text{LLPE}[j, : ]}_2^2 - d_r(i, j)^2 \Big| \leq \epsilon}

To begin, let $d = n$ such that $\mbf{P}_\text{LLPE} = \mbf{U}\mbf{W}_{\text{LLPE}} \in \mbb{R}^{n \times n}$. Here, column $j$ of $\mbf{W}_{\text{LLPE}}$ corresponds to a Chebyshev series parameterized by $\bs{\theta}_j \in \mbb{R}^M$. Now, the key in proving the approximation is choose $M$ large enough such that $h(\lambda; \bs{\theta}_j)$ can approximate functions of the form, 
\eq{f_j(\lambda) = r(\lambda_j) e^{-(\lambda-\lambda_j)^2 C_\text{max}},}

where $\lambda \in [\lambda_{\text{min}}, \lambda_{\text{max}}]$ and $C_\text{max}$ is a large constant. Let us discuss function $f_j$. Essentially, $f_j$ is a function that activates at $\lambda_j$ with value $r(\lambda_j)$ and is 0 for all other eigenvalues. Indeed if we choose $C_\text{max}$ to be large enough, we find that only values within a range $[\lambda_j-\delta, \lambda+\delta]$ for some small $\delta>0$ have nonzero output, while values outside this range are zero. Notice further that $f_j$ has infinitely many derivatives, and thus if we choose a large enough $M$, by Theorem \ref{thm:chb} we have $\nrm{f(\lambda) - h(\lambda; \bs{\theta}_j)} < \epsilon$ for some prescribed $\epsilon$. In our case, we choose a large enough $M$ such that $\nrm{f(\lambda) - h(\lambda; \bs{\theta}_j)} < \epsilon/n\cdot\text{max}_k(\mbf{u}_k[i]-\mbf{u}_k[j])^2$. Having established that $h(\lambda; \bs{\theta}_j)$ can approximate $f_j(\lambda)$, $\mbf{P}_\text{LLPE} = \mbf{U}\mbf{W}_{\text{LLPE}} \approx \mbf{U} r(\bs{\Lambda})$. This gives us, 
\al{&\Big| \nrm{\mbf{P}_\text{LLPE}[i, : ] - \mbf{P}_\text{LLPE}[j, : ]}_2^2 - d_r(i, j)^2 \Big| \\ 
&\leq \sum_{k=1}^n (r(\lambda_k) + \frac{\epsilon}{n\cdot\text{max}_k (\mbf{u}_k[i]-\mbf{u}_k[j])^2}) (\mbf{u}_k[i] - \mbf{u}_k[j])^2 - \sum_{k=1}^n (r(\lambda_k) (\mbf{u}_k[i] - \mbf{u}_k[j])^2 \\
&\leq \sum_{k=1}^n (\frac{\epsilon}{n\cdot\text{max}_k (\mbf{u}_k[i]-\mbf{u}_k[j])^2}) (\mbf{u}_k[i] - \mbf{u}_k[j])^2 \leq \epsilon}

\end{proof}

\begin{proposition}
Let $\mbf{A}$ and $\mbf{L}$ be the Adjacency and Laplacian matrix drawn from the stochastic block model $G(n, k, p, q)$. If $p \gg q$ or $q \gg p$ and $\text{min}(d_i) \geq C \text{ln}(n)$ where $C$ is an appropriately large constant, then LLPE can correctly recover the true communities up to an orthogonal transformation with at most $\mc{O}(k^{\frac{3}{2}})$ misclassified nodes.
\end{proposition}

\begin{proof}
First assume $p \gg q$ and let $0 < \epsilon < \mc{O}(k^\frac{3}{2})$. Formally, we aim to show,
\eq{\nrm{\mbf{P}_\text{LLPE} - \mbf{U}_k}_F \leq \epsilon.}

To begin, let $d=k$. The key in proving the approximation is to choose $M$ large enough such that $h(\lambda; \bs{\theta}_j)$ can approximate functions of the form,
\eq{f_j(\lambda) = e^{-(\lambda - \lambda_j)^2C_{\text{max}}}}

Here, $f_j$ is a function that activates at $\lambda_j$ with value $+1$ and is zero for all other eigenvalues. Since again $f_j$ has infinitely many derivatives, we can choose a large enough $M$ such that by Theorem \ref{thm:chb} we have $\nrm{f(\lambda) - h(\lambda; \bs{\theta}_j)} < \epsilon$ for some prescribed $\epsilon/k$. Having established that $h(\lambda; \bs{\theta}_j)$ can approximate $f_j(\lambda)$, $\mbf{P}_\text{LLPE} \approx \mbf{U}_k$, thus proving the result for $p \gg q$. We can similarly prove the result assuming $q \gg p$, by defining $f_j$ as follows,
\eq{f_j(\lambda) = e^{-(\lambda - \lambda_{n-j})^2C_{\text{max}}}}

where instead $f_j$ is a function that activates at $\lambda_{n-j}$ with value $+1$ and is zero elsewhere. This allows us to effectively obtain, $\mbf{P}_\text{LLPE} \approx \mbf{U}_{-k:}$, where $\mbf{U}_{-k:}$ is the last $k$ eigenvectors of the graph Laplacian, thus proving the result for $p \gg q$.

\end{proof}

\subsection{Generalization Properties of LLPE \label{pf:llpe_gen}}

We provide the relevant background on Rademacher complexity found in \citet{awasthi2020rademacher}. Let $\mc{H}$ be a hypothesis class where $h: \mbb{R}^d \to \mbb{R}$ when $h \in \mc{H}$. The empirical Rademacher complexity of $\mc{H}$ given sample $S=(\mbf{x}_1, \ldots, \mbf{x}_n)$ is defined, 
\eq{\hat{\mf{R}}(\mc{H}) = \mbb{E}_{\bs{\sigma}} \brt{\text{sup}_{h\in\mc{H}} \frac{1}{m} \sum_{i=1}^n \sigma_i h(\mbf{x}_i)}}

where $\bs{\sigma}$ is a vector of Rademacher variables drawn uniformly from values in $\{-1, +1\}$. For a family of functions $\mc{F}$ where $f: \mbb{R}^d \to [0, 1]$, we have the following result: for any $\delta > 0$, with probability $1-\delta$ over the randomness of $S \sim D$, the following inequality holds for all $f\in \mc{F}$:
\eq{\mbb{E}_{x\sim D} [f(x)] - \mbb{E}_{x\sim S}[f(x)] \leq 2 \hat{\mf{R}}(\mc{F}) + \sqrt{\frac{\text{log}(\frac{1}{\delta})}{2m}}}

\citet{awasthi2020rademacher} considers the hypothesis set of linear predictors with weight bounded in $l^p$ norm with $\mc{H}_p = \{\mbf{x} \to \mbf{x} \cdot \mbf{w} : \nrm{\mbf{w}}_p \leq W\}$ and proves the following upper and lower bounds on the empirical Rademacher complexity of $\mc{H}_2$.
\begin{theorem}[\citet{awasthi2020rademacher}]
Let $\mc{H}_2 = \{\mbf{x} \to \mbf{x} \cdot \mbf{w} : \nrm{\mbf{w}}_2 \leq W\}$ be a hypothesis class of linear predictors defined over $\mbb{R}^d$ with bounded weight in $l^2$ norm. Then, the empirical Rademacher complexity of $\mc{H}_p$ for sample $S=(\mbf{x}_1, \ldots, \mbf{x}_n)$ admits the following upper and lower bounds, 
\eq{\frac{W}{\sqrt{2}n} \nrm{\mbf{X}}_{F} \leq \hat{\mf{R}}_S(\mc{H}_2) \leq \frac{W}{n} \nrm{\mbf{X}}_{F}}
where $\mbf{X}$ is the $n \times d$ feature matrix and $\nrm{\cdot}_F$ is the Frobenius norm.
\label{thm:lin}
\end{theorem}

\begin{theorem}
Let $\mc{H}_{\text{LLPE}} = \{\tilde\lambda \to \sum_{m=1}^M \theta_m \cdot \tilde T_m(\tilde \lambda) : \bs{\theta} \in \mbb{R}^M, \nrm{\bs{\theta}}_2 \leq C_{\text{LLPE}}\}$, where $C_{\text{LLPE}}$ is some constant greater than 0, $\tilde\lambda$ is the rescaled eigenvalues that lie in $[-1, +1]$, and $\tilde T_m$ is the normalized Chebyshev polynomial. Then, the empirical Rademacher complexity of the hypothesis class $\mc{H}_{\text{LLPE}}$ for a sample $S = (\lambda_0, \ldots, \lambda_n)$ admits the following upper and lower bounds, 
\eq{\frac{C_{\text{LLPE}}}{\sqrt{2n}} \leq \hat{\mf{R}}_{S}(\mc{H}_{\text{LLPE}}) \leq \frac{\sqrt{2}C_{\text{LLPE}}}{\sqrt{n}}}
\end{theorem}

\begin{proof}
The key in proving Theorem \ref{thm:rad} is to recognize that $\mc{H}_{\text{LLPE}}$ is equivalent to a linear hypothesis class with weight $\bs{\theta}$ and feature matrix $\tilde{\mbf{T}} \in \mbb{R}^{n \times M}$, where $\tilde{ \mbf{T}}$ is a matrix of normalized Chebyshev polynomial outputs. Once we characterize the Frobenius norm of $\tilde{ \mbf{T}}$, we can apply Theorem \ref{thm:lin} to obtain our upper and lower bounds. To begin, we can represent the Frobenius norm of $\tilde{ \mbf{T}}$ as follows, 
\eq{\nrm{\tilde{\mbf{T}}}_F = \nrm{\begin{pmatrix}
\tilde T_0(\lambda_0) & \cdots &  \tilde T_M(\lambda_0) \\
\vdots & \ddots & \vdots \\
\tilde T_0(\lambda_n) & \cdots &  \tilde T_M(\lambda_n) \\
\end{pmatrix}}_F }

Notice that the $j$th column of $\tilde{\mbf{T}}$ corresponds to the vector of the Chebyshev outputs on the eigenvalues where the Chebyshev polynomial is of order $j$. We can then upper bound the Frobenius norm of $\tilde{\mbf{T}}$ by replacing each column with the vector of the $l^\infty$ norms of the Chebyshev polynomials as follows, 
\eq{\nrm{\begin{pmatrix}
\tilde T_0(\lambda_0) & \cdots &  \tilde T_M(\lambda_0) \\
\vdots & \ddots & \vdots \\
\tilde T_0(\lambda_n) & \cdots &  \tilde T_M(\lambda_n) \\
\end{pmatrix}}_F \leq
\nrm{\begin{pmatrix}
\nrm{\tilde T_0}_\infty & \cdots &  \nrm{\tilde T_M}_\infty \\
\vdots & \ddots & \vdots \\
\nrm{\tilde T_0}_\infty & \cdots &  \nrm{\tilde T_M}_\infty \\
\end{pmatrix}}_F}

We can now apply Theorem \ref{thm:nrm} to the upper bound on the Frobenius norm of $\tilde{\mbf{T}}$ , where we obtain, 
\eq{
\nrm{\begin{pmatrix}
\nrm{\tilde T_0}_\infty & \cdots &  \nrm{\tilde T_M}_\infty \\
\vdots & \ddots & \vdots \\
\nrm{\tilde T_0}_\infty & \cdots &  \nrm{\tilde T_M}_\infty \\
\end{pmatrix}}_F =
\nrm{\begin{pmatrix}
1 & \cdots &  2^{1-m} \\
\vdots & \ddots & \vdots \\
1 & \cdots & 2^{1-m} \\
\end{pmatrix}}_F
}

Now, we can proceed to derive an upper bound on the Frobenius norm of $\tilde{\mbf{T}}$ as follows, 
\al{\nrm{\tilde{\mbf{T}}}_F &\leq \sqrt{\sum_{j=1}^M \sum_{i=1}^n \abs{\tilde{\mbf{T}}[i, j]}^2} \\
&\leq \sqrt{\sum_{j=1}^M n ({2^{1-j}})^2} \\
&\leq \sqrt{n} \sqrt{\sum_{j=0}^\infty {2^{-j}}} \label{eq:57} \\
&= \sqrt{2n} \label{eq:58}
}
where Equation \ref{eq:57} follows since $(2^{1-j})^2 \leq 2^{1-j}$ for all $j \geq 0$ and Equation \ref{eq:58} follows since $\sqrt{\sum_{j=0}^\infty {2^{-j}}}$ is a geometric series. Having derived an upper bound on the Frobenius norm of $\tilde{ \mbf{T}}$, applying Theorem \ref{thm:lin} to $\mc{H}_{\text{LLPE}}$ proves the upper bound of the result. We prove the lower bound of the result by lower bounding the Frobenius norm of $\tilde{\mbf{T}}$. To do so, we replace each column with the vector of the minimum absolute value of the Chebyshev polynomials as follows, 
\eq{
\nrm{\begin{pmatrix}
\tilde T_0(\lambda_0) & \cdots &  \tilde T_M(\lambda_0) \\
\vdots & \ddots & \vdots \\
\tilde T_0(\lambda_n) & \cdots &  \tilde T_M(\lambda_n) \\
\end{pmatrix}}_F \geq
\nrm{\begin{pmatrix}
\text{min}_{\lambda}(\abs{\tilde T_0(\lambda)}) & \cdots &  \text{min}_{\lambda}(\abs{\tilde T_M(\lambda)}) \\
\vdots & \ddots & \vdots \\
\text{min}_{\lambda}(\abs{\tilde T_0(\lambda)}) & \cdots &  \text{min}_{\lambda}(\abs{\tilde T_M(\lambda)}) \\
\end{pmatrix}}_F
}

The Chebyshev polynomial $\tilde T_0$ is defined as the constant function equal to one, while each of the Chebyshev polynomials of order $m>0$ have exactly $m+1$ zeros. Thus, the lower bound of the Frobenius norm of $\tilde{\mbf{T}}$ can be expressed as the following,
\eq{
\nrm{\begin{pmatrix}
\text{min}_{\lambda}(\abs{\tilde T_0(\lambda)}) & \cdots &  \text{min}_{\lambda}(\abs{\tilde T_M(\lambda)}) \\
\vdots & \ddots & \vdots \\
\text{min}_{\lambda}(\abs{\tilde T_0(\lambda)}) & \cdots &  \text{min}_{\lambda}(\abs{\tilde T_M(\lambda)}) \\
\end{pmatrix}}_F 
=\nrm{\begin{pmatrix}
1 & 0 & \cdots &  0 \\
\vdots & \vdots & \ddots & \vdots \\
1 & 0 & \cdots & 0 \\
\end{pmatrix}}_F
}
The lower bound of the Frobenius norm of $\tilde{\mbf{T}}$ then becomes,
\eq{\nrm{\tilde{\mbf{T}}}_F \geq \sqrt{n}}

Lastly, applying Theorem \ref{thm:lin} to $\mc{H}_{\text{LLPE}}$ proves the lower bound of the result.

\end{proof}

\section{ADDITIONAL THEORY, EMPIRICAL RESULTS AND DISCUSSIONS \label{sec:add}}

\subsection{Comparing LLPE and LPE-FLK with Complex Synthetic Graphs \label{sec:power}}

We conducted an additional experiment on synthetic graphs of varying homophily levels using a modified preferential attachment process introduced by \citet{zhu2020beyond} (Figure \ref{fig:power}). Here, nodes are added one by one, and the probability that a new node $u$ of class $i$ forms an edge with existing node $v$ of class $j$ is proportional to both the class compatibility between classes $i$ and $j$, $H_{i,j}$, and the degree of node $v$. As a result, the degree distribution for the resulting graphs follows a powerlaw, and the homophily can be controlled by the compatibility matrix $H$. We train and evaluate a GT equipped with LPE-FLK and LLPE and present results below. The results demonstrate that GT with LPE-FLK obtains the same performance across different homophilies and is unable to capture the relevant graph structure in high or low homophily settings on the more complex powerlaw graph. On the other hand, the performance of GT with LLPE significantly increases when homophily is either very high or very low, indicating it is able to capture the relevant graph structure and identify the relevant eigenvectors in complex graph structures.

\begin{figure*}[h!]
    \centering
    \includegraphics[width=.75\textwidth]{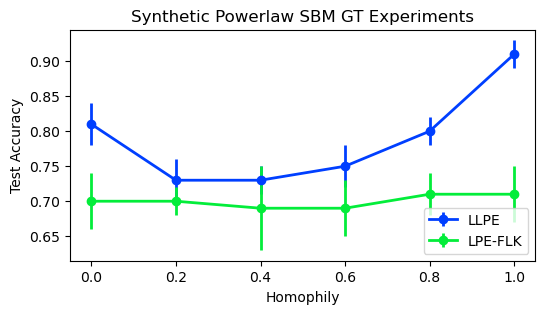}
    \caption{Mean and standard deviations (error bars) of all model-PE combinations on the synthetic SBMs. LLPE performs well across both high homophily and high heterophily, while LPE-FK does not.}
    \label{fig:power}
\end{figure*}

\begin{figure*}[h!]
     \centering
     \begin{subfigure}[b]{.48\textwidth}
         \centering
         \includegraphics[width=\textwidth]{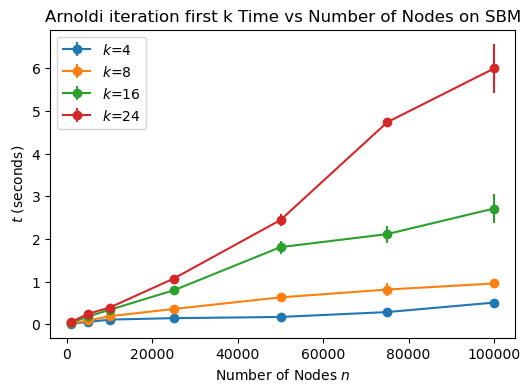}
         \caption{Time complexity in obtaining first $k$ eigenvectors}
     \end{subfigure}
     \hfill
     \begin{subfigure}[b]{.495\textwidth}
         \centering
         \includegraphics[width=\textwidth]{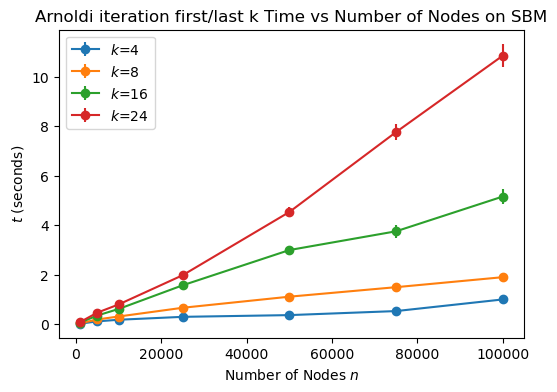}
         \caption{Time complexity in obtaining first and last $k$ eigenvectors}
     \end{subfigure}
    \caption{Arnoldi-iteration time complexity along SBMs varying $n$ and $k$. }
    \label{fig:time}
\end{figure*}

\begin{figure*}[h!]
     \centering
     \begin{subfigure}[b]{.48\textwidth}
         \centering
         \includegraphics[width=\textwidth]{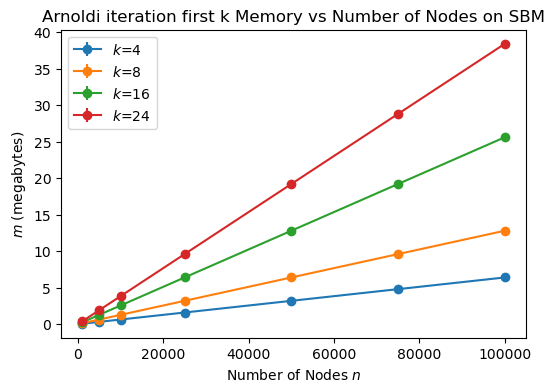}
         \caption{Space complexity obtaining first $k$ eigenvectors}
     \end{subfigure}
     \hfill
     \begin{subfigure}[b]{.495\textwidth}
         \centering
         \includegraphics[width=\textwidth]{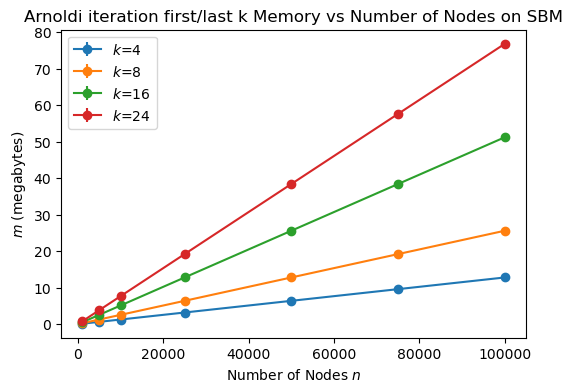}
         \caption{Space complexity obtaining first and last $k$ eigenvectors}
     \end{subfigure}
    \caption{Arnoldi-iteration space complexity along SBMs varying $n$ and $k$. }
    \label{fig:space}
\end{figure*}

\begin{figure*}[h!]
    \centering
    \includegraphics[width=.75\textwidth]{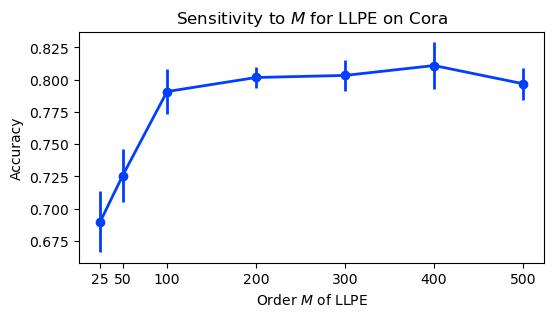}
    \caption{Sensitivity analysis to the order $M$ of LLPE on Cora.}
    \label{fig:ord}
\end{figure*}

\subsection{Computational Efficiency of LLPE \label{sec:eff}}

We demonstrate the computational efficiency of the large version of LLPE by measuring both time and space complexity of Arnoldi iteration in obtaining the first and last $k$ eigenvectors on an SBM of sizes up to 100,000 nodes. We utilize a fast implementation of Arnoldi iteration readily available in SciPy optimized with sparse matrix and vector operations. We conduct our experiment on an Intel ® Xeon ® CPU E5-2620 measuring the runtime and memory usage of the algorithm as we increase $n$, the size of the graph and $k$, the number of eigenvectors obtained (Figures \ref{fig:time} and \ref{fig:space}). Figure \ref{fig:time} demonstrates that obtaining the first and last $k$ eigenvectors of graphs with 1000s of nodes is very fast, requiring only seconds, while obtaining the first and last $k$ eigenvectors on graphs of size 100,000 nodes requires tens of seconds. Figure \ref{fig:space} demonstrates that the memory requirements during Arnoldi-iteration is also efficient, requiring tens of megabytes on graphs of size 100,000 nodes.

\subsection{Sensitivity to Order $M$ of LLPE \label{app:ord}}

In order to test the sensitivity of LLPE to choices of $M$, we plot the test performance of LLPE across different choices of $M$ on Cora's giant component in Figure \ref{fig:ord}. Figure \ref{fig:ord} tells us that as long as $M$ is large enough, LLPE obtains high test accuracy. Interestingly, the sensitivity analysis agrees with both our theoretical results on the approximation capabilities and statistical generalization for LLPE. In particular, Proposition \ref{prp:com} and Theorem \ref{thm:dst} indicate that $M$ needs to be large enough in order for LLPE to obtain good approximation capabilities. Similarly, we find empirically that when $M$ is too small as indicated when $M\in [25, 50]$ performance degrades. On the other hand, Theorem \ref{thm:rad} tells us that LLPE's statistical generalization does not depend explicitly on the order $M$ of LLPE. As a result, we can select a large $M$ to guarantee high expressivity, while not harming generalization. Indeed, we find empirically that for choices of $M$ as large as $M=500$, LLPE maintains a high test accuracy. 

\subsection{Sensitivity to Number of Eigenvectors $k$ for Large LLPE \label{app:neigs}}

In order to test the sensitivity of the large version of LLPE to choices of $k$, we plot the performance of LLPE across different choices of $k$ on Cora's giant component in Figure \ref{fig:k}. Figure \ref{fig:k} tells us that if $k$ is too small such as $k \in[8, 32]$, then performance degrades. On the other hand, once $k$ exceeds this range, performance is stable across choices of $k$. Although performance is stable across choice of $k>32$, we find that the best performing choice of $k$ lies at $k=64$ rather than $k=1024$. This result indicates that we can potentially improve LLPE by searching across choices of $k$ from the start and ends of the spectrum. However, this design introduces another hyperparameter, requiring a careful search across the spectrum of the Laplacian. It is also prone to missing eigenvectors lying in the middle of the spectrum, where previous work has shown these eigenvectors can be relevant \citep{li2023restructuring}. We thus leave this idea for exploration in future work.

Interestingly, the sensitivity to the choice of $k$ is different on LLPE in comparison to LPE. In LPE, the best choice of $k$ tends to be small where $k\in[8, 32]$, while for LLPE the best choice of $k$ tends to be large where $k\in[64, 1024]$. In fact for LPE, we observe in our empirical results that no choices of $k$ greater than 32 obtain the best performance among models trained with LPE. We hypothesize that these behaviors are due to the ability of LLPE to learn which eigenvectors are important and inability of LPE to learn which eigenvectors are important. Since LLPE learns which eigenvectors are important, providing them with only a small subset of the eigenvectors reduces performance. On the other hand, LPE cannot learn which eigenvectors are important, and thus providing to them more eigenvectors reduces performance.

\begin{figure*}[h!]
    \centering
    \includegraphics[width=.75\textwidth]{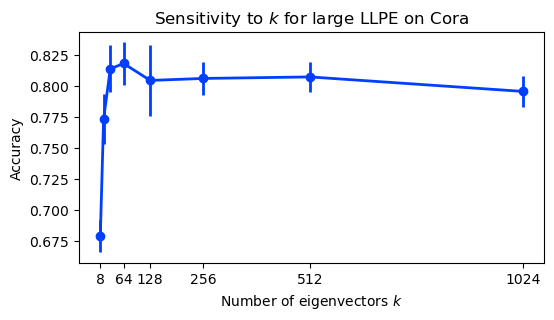}
    \caption{Sensitivity analysis to the number of eigenvectors $k$ of large LLPE on Cora's giant component.}
    \label{fig:k}
\end{figure*}

\subsection{Additional related work on PEs \label{sec:learn_pes}}

RWSE \citep{dwivedi2022graph} and SAN-PE \citep{kreuzer2021rethinking} focus on improving GNN expressivity for graph classification by using learned MLPs or Set Transformers based on the first $k$ eigenvectors. SignNet \citep{lim2022sign} addresses LPE's sign ambiguity, and PEG \citep{wang2022equivariant} designs rotation and reflection equivariant LPEs.

PEs other than LPE rely on random walks and node distances such as RWSE, the diagonals of the $m$-step random walk matrix \citep{dwivedi2022graph}, pair-wise shortest path node distances \citep{ying2021transformers}, and shortest paths between nodes and anchor nodes \citep{you2019position}. \citet{park2022deformable} propose applying MLPs to the matrix of Katz indices, a weighted sum of powers of the adjacency matrix. This design is similar to RWSE \citep{dwivedi2022graph} where they propose to apply MLPs to the random walk matrix. \citet{bo2023specformer} propose Specformer, a new spectral GNN, that applies sine and cosine PEs to the eigenvalues, feeds them to a transformer and MLP, and finally uses the processed eigenvalues as spectral filters in a spectral GNN. \citet{li2024mpformer} propose a PE that applies message-passing and sine and cosine PEs to node representations. This design is similar to the work of structure-aware transformers \citep{chen2022structure}. Lastly, \citet{pan2023beyond} propose high and low-pass filters for graph clustering. In contrast to the above works, we design a new position encoding that operates directly on the full matrix of eigenvectors without the use of neighborhood information or message passing. 

\citet{arnaiz2022diffwire} propose a PE that can predict commute times and the Fiedler vector given node features. The PE is then shown to improve performance on small heterophilous benchmarks. \citet{velingker2023affinity} propose PEs based on random walks such as effective resistance, commute, and hitting times. This work is also similar to RWSE \citep{dwivedi2022graph}. \citet{zhang2023rethinking} propose to include graph distances as PEs in order to solve graph biconnectivity. Lastly, \citet{lu2024representation} propose a new class of Laplacian matrices that reduce diffusion distance. In contrast to the above PEs, which focus on non-learnable random walk measures, we propose LLPE, a learnable PE that extends Laplacian PEs to capture homophily and heterophily. 

\section{EXPERIMENTAL DETAILS \label{app:exp}}

\subsection{Synthetic Experimental Details}

In our synthetic experiments, we utilize binary and multiclass SBMs. For each dataset, we set the average degree of each node as $d=10$ and generate 5 different SBMs according to the homophily ratios $h=[0.0, 0.2, 0.4, 0.6, 0.8, 1.0]$. For example, when the homophily ratio is $0.8$, each node has on average 8 neighbors in their own community and 2 neighbors in other communities. Node labels are the node's community. For the binary SBM, we sample 10 independent Gaussian distributed node features with mean $\mbf{y}[i]\cdot\mu$ and variance $\sigma^2$. For the multiclass SBM, we sample a multivariate Gaussian distributed vector with mean $\mu \cdot \text{one-hot}(\mbf{y}[i])$ and covariance $\Sigma$. For all model and PE combinations, we follow the same hyperparameter search as described in the real-world experiments.

\subsection{Real-world Experimental Details}

We use small, medium, and large datasets from \citep{yang2016revisiting, bojchevski2018deep, shchur2018pitfalls, Pei2020Geom-GCN, platonov2023critical, lim2021large}. We test the following base models: MLP, a feedforward neural network, SAGE \citep{hamilton2017inductive}, a message-passing neural network, GT (full) \citep{dwivedi2021generalization}, a graph transformer that leverages global attention. For each model, we search across learning rates $\eta \in [0.05, 0.01, 0.005, 0.001]$, optimizer SGD with momentum and Adam, dropouts in [0.0, 0.2, 0.5], size of hidden dimension in [64, 128], and number of layers in [1, 2]. For the GT, we additionally search across layer norms in [0.0, 0.0001]. For all models, we train with full batch mode with early stopping set to 200. We select the best performing model on the validation set for evaluation. In addition to LLPE, we test all models with the following PEs.
\begin{enumerate}
    \item LPE-FK: The first $k$ nontrivial eigenvectors of the graph Laplacian as defined in \citet{dwivedi2021generalization} with a learnable linear projection matrix.
    \item LPE-FLK: The first and last $k$ eigenvectors of the graph Laplacian with a learnable linear projection matrix.
    \item LPE-Full: All eigenvectors of the graph Laplacian with a learnable linear projection matrix.
    \item Elastic-PE: The matrix of electrostatic potentials as defined in \citet{liu2023graph} with a learnable linear projection matrix.
    \item SignNet: SignNet with DeepSets as defined in \citet{lim2022sign}.
    \item SAN-PE: SAN-PE as defined in \citet{kreuzer2021rethinking}. 
    \item RWSE: The diagonals of the $m$-step random walk matrix as defined in \citep{dwivedi2022graph}.
\end{enumerate}

For LPE-FK, LPE-FLK, Elastic-PE, SignNet, and SAN-PE, we search across the number of eigenvector and eigenvalue pairs $k\in[8, 16, 32, 64, 128, 256, 512]$. For all PEs, we find that typically the best performing $k$ lies in the range [8, 32]. For RWSE, we search across the range $m\in[8, 16, 32, 64]$. For LLPE, we search across $M\in [64, 128]$, $l^1$ regularization in [0.001, 0.0001, 0.0], and set $d = 128$. For all model and PE combinations, we linearly project the node features and PE representations into the same space, then concatenate them before passing the entire representation to the base model as described in \citet{dwivedi2021generalization}. We implement and train our models on a GeForce GTX 1080.



\end{document}